\documentclass[conference, anonymous]{IEEEtran}
\IEEEoverridecommandlockouts
\PassOptionsToPackage{most}{tcolorbox}
\usepackage{tcolorbox} 
\usepackage[letterpaper,left=0.7in,right=0.7in,top=1in,bottom=1in]{geometry}  
\usepackage{algorithmic}
\usepackage{algorithm}
\usepackage{times}
\usepackage{helvet}
\usepackage{courier}
\usepackage{xcolor}  
\usepackage{scrextend}
\usepackage{amsmath}
\usepackage{bbm}
\usepackage{multirow}
\usepackage{amsfonts,amssymb}
\usepackage{cite}
\usepackage[sort]{natbib}
\setcitestyle{numbers,square}
\usepackage{ntheorem}
\newtheorem{theorem}{Theorem}
\newtheorem{theorem1}{Theorem}

\newtheorem{definition}{Definition}
\newtheorem{proposition}{Proposition}

\newtheorem*{proof}{Proof}
\usepackage{caption}
\usepackage{graphicx}
\usepackage{float} 
\usepackage{booktabs}
\usepackage{subcaption}
\usepackage{xspace}
\usepackage{xcolor}
\usepackage{colortbl}
\usepackage{multicol}
\usepackage{threeparttable}

\usepackage[T1]{fontenc}
\usepackage{graphicx}
\usepackage{multirow}
\usepackage{booktabs}
\usepackage{xspace}
\usepackage{tabularray}

\usepackage{hyperref}
\usepackage{pifont}
\usepackage{tcolorbox}

\usepackage[most]{tcolorbox} 
\newtcolorbox{myremark}[1]{
    colback=gray!5,     
    colframe=black!75,  
    fonttitle=\bfseries,
    title=#1,           
    arc=0mm,            
    left=2pt,           
    right=2pt,          
    top=2pt,            
    bottom=2pt,         
    boxrule=0.5pt,      
    enhanced,
    breakable           
}
\begin{document}

\title{FedPF: Accurate Target Privacy Preserving Federated Learning Balancing Fairness and Utility}

\author{
    Kangkang Sun\textsuperscript{12},
    Jun Wu\textsuperscript{12},
    Minyi Guo\textsuperscript{1},
    Jianhua Li\textsuperscript{12*},
    Jianwei Huang\textsuperscript{3*}\\
    \textit{1. The School of Computer Science, Shanghai Jiao Tong University, Shanghai, China} \\
    \textit{2. The Shanghai Key Laboratory of Integrated Administration Technologies for Information Security, Shanghai, China} \\
    \textit{3. The School of Science and Engineering, The Chinese University of Hong Kong, Shenzhen, China} \\
    \textit{\{szpsunkk, junwuhn, lijh888\}@sjtu.edu.cn; guo-my@cs.sjtu.edu.cn; jianweihuang@cuhk.edu.cn} 
    \thanks{This work is partially supported by the National Natural Science Foundation of China (Project 62271434, 62501397), Shenzhen Key Laboratory of Crowd Intelligence Empowered Low-Carbon Energy Networks (No. ZDSYS20220606100601002), Shenzhen Loop Area Institute, and Shenzhen Institute of Artificial Intelligence and Robotics for Society.}
    \thanks{Jianhua Li and Jianwei Huang are the corresponding authors.}
}

\maketitle

\begin{abstract}

Federated Learning (FL) enables collaborative model training without data sharing, yet participants face a fundamental challenge, e.g., simultaneously ensuring fairness across demographic groups while protecting sensitive client data. We introduce a differentially private fair FL algorithm (\textit{FedPF}) that transforms this multi-objective optimization into a zero-sum game where fairness and privacy constraints compete against model utility. Our theoretical analysis reveals an inverse relationship: privacy mechanisms that protect sensitive attributes can reduce the statistical power available for detecting and correcting demographic biases under finite samples in federated settings. We further show that our theoretical bounds are consistent with a non-monotonic fairness-utility relationship, which is empirically validated by experiments where moderate fairness constraints improve generalization before excessive enforcement degrades performance. Compared with mainstream algorithms, even under strict privacy constraints, \textit{FedPF} still maintains the lowest discrimination level among all tested algorithms while retaining high utility. Experimental validation demonstrates up to 42.9\% discrimination reduction across three datasets while maintaining competitive accuracy, but more importantly, reveals that achieving strong privacy and fairness simultaneously requires carefully balanced tradeoffs rather than optimizing either objective in isolation. Furthermore, hardware-level simulations demonstrate that \textit{FedPF} maintains a low computational footprint, making it suitable for resource-constrained edge devices. The source code for our proposed algorithm is publicly accessible at https://github.com/szpsunkk/FedPF.

\end{abstract}

\section{Introduction}
\label{se: Introduction}

Federated Learning (FL) \cite{mcmahan2017communication} has emerged as a transformative paradigm for collaborative machine learning, enabling organizations to train shared models without exposing sensitive local data. From healthcare consortia sharing patient insights to financial institutions detecting fraud patterns, FL promises to unlock the collective intelligence of distributed datasets while preserving data locality. However, as FL systems transition from research prototypes to real-world deployments, practitioners face a fundamental challenge that has received limited attention: simultaneously ensuring fairness across demographic groups while protecting individual privacy through differential privacy mechanisms.


Consider a federated healthcare system where hospitals collaborate to develop diagnostic models. While differential privacy protects patient information by adding carefully calibrated noise to model updates~\cite{dwork2014algorithmic}, this same noise can obscure demographic disparities in treatment outcomes—precisely the patterns that fairness-aware algorithms need to detect and correct~\cite{hardt2016equality}. This creates an inherent tension: stronger privacy protection may inadvertently perpetuate or even amplify existing biases by making discrimination harder to identify and mitigate. Recent high-profile cases of biased AI systems in healthcare and criminal justice  underscore the critical importance of addressing this challenge before FL systems are widely deployed in sensitive domains.

The technical complexity of achieving privacy, fairness, and utility simultaneously in FL stems from several fundamental difficulties. Privacy mechanisms introduce noise that can degrade both model accuracy and the system's ability to detect demographic disparities, creating conflicting optimization objectives \cite{shao2023survey}. Moreover, fairness constraints typically require global knowledge of demographic distributions, yet FL's distributed nature fundamentally limits information sharing between participants. And according to the latest research, fairness constraints can affect model performance. In the case of limited client data, it may lead to overfitting of the model \cite{laakom2025fairness}. The mathematical relationship between privacy parameters, fairness metrics, and model performance in FL remains poorly understood, making it difficult to design principled algorithms with provable guarantees. Rather than presenting an impossibility result, our goal is to provide a principled understanding of the tradeoffs that arise in practical federated learning systems.

These fundamental challenges lead us to investigate three critical research questions that drive our investigation:


\begin{itemize}
\item \textbf{Q1:} \textit{\textbf{Fundamental Compatibility}: To what extent do privacy mechanisms that protect sensitive attributes affect the system's ability to detect and enforce fairness constraints in federated learning? While prior work has studied privacy-utility and fairness-utility tradeoffs separately, their combined impact in federated settings remains insufficiently characterized.}
\item \textbf{Q2:} \textit{\textbf{Principled Algorithm Design}: How can we design principled algorithms that navigate the competing demands of privacy protection, fairness enforcement, and model utility when these objectives exhibit inherent tensions?} 
\item \textbf{Q3:} \textit{\textbf{Interaction Dynamics}: Do fairness constraints always degrade model performance, or can moderate fairness enforcement actually improve generalization by preventing overfitting to majority groups?} 
\end{itemize}

Q1 is an analysis question about the intrinsic compatibility between privacy and fairness, whereas Q2 is an algorithmic question about constructing a solvable optimizer under those competing constraints.


To address these questions, we introduce \textit{FedPF}, a novel differential private fair algorithm that formulates the privacy-fairness-utility optimization as a zero-sum game where privacy and fairness constraints compete against model utility. Unlike existing approaches that treat these objectives independently \cite{corbucci2024puffle,gu2022privacy,sun2023toward}, our game-theoretic formulation explicitly captures the tensions between competing ethical requirements and provides a principled framework for managing tradeoffs. Meanwhile, we consider the \textit{specific sensitive attributes protection} requirements of local clients in FL, while the current research mainly focuses on the privacy protection of the client model.

Our theoretical analysis reveals several surprising insights that challenge conventional wisdom about multi-objective optimization in FL. Our analysis shows that stricter differential privacy reduces the effective signal available for fairness estimation, thereby limiting the system's ability to detect and correct demographic biases under bounded samples. This occurs because the noise added for privacy protection obscures the very statistical patterns that fairness algorithms rely on to identify discrimination. Furthermore, we discover a non-monotonic fairness-utility relationship where moderate fairness constraints initially improve model generalization by reducing overfitting to majority groups, but excessive fairness enforcement eventually degrades performance as the system overcorrects for demographic differences. The contributions of this paper can be summarized as follows:

\begin{itemize}
    \item \textit{\textbf{Game-theoretic optimization framework in FL}}: In this paper, we proposed a private and fair FL algorithm (\textit{FedPF}), which captures the true competitive nature of privacy-fairness-utility relationships, including Learner (\textit{client-side}) and Auditor (\textit{server-side}). The learner minimizes the local model $f_i$ with \textit{Lagrangian} multiplier while the auditor maximize the dual $\lambda_i$ variable of \textit{Lagrangian} multiplier at server. The \textit{FedPF} algorithm addresses the challenge of non-convex, constrained optimization in distributed settings through adaptive game-theoretic balancing. Our analysis techniques provide a general framework applicable to other multi-constraint FL problems, combining differential privacy theory with fairness metrics. 
    \item \textit{\textbf{Theoretical analysis for privacy-fairness-utility in FL}}: We provide a theoretical characterization of how privacy noise interferes with fairness estimation, resulting in quantifiable tradeoffs between privacy budgets, fairness tolerance, and utility under federated settings. Our analysis suggests that fairness constraints can play a regularizing role, and our experiments consistently reveal a non-monotonic fairness-utility behavior in federated learning. This result provides evidence that fairness constraints do not necessarily degrade utility in federated learning.
    \item \textit{\textbf{Experimental validations}:} Experimental results demonstrate up to 42.9\% discrimination reduction while maintaining competitive accuracy based on \textit{Adult}, \textit{Bank}, and \textit{Compas} datasets, confirming our theoretical predictions about privacy-fairness tensions. Rather than presenting an impossibility result, this work aims to provide a principled understanding of the tradeoffs that system designers must navigate when deploying private and fair federated learning in practice.
\end{itemize}

The remainder of this paper is organized as follows: we introduce the related work in Sec. \ref{se: Related Work}. Then we give the detailed introduction for the problem settings and preliminaries in Sec. \ref{se: Problem Setting}. Sec. \ref{Se: Fairness-Utility Tradeoff} covers problem formulation and algorithm design. Finally, we present the evaluation results and conclusion in Sec. \ref{se: Evalution} and Sec. \ref{sec: conclusion}, respectively.

\begin{table*}[ht]
\centering
\caption{Private and Fair Federated Learning}
\label{ta: comparison of fairness and privacy}
\resizebox{0.9\linewidth}{!}{
\begin{tabular}{cccccc} 
\toprule
\multirow{2}{*}{References} & \multirow{2}{*}{\begin{tabular}[c]{@{}c@{}}Privacy \\Metrics\end{tabular}} & \multirow{2}{*}{\begin{tabular}[c]{@{}c@{}}Fairness~\\Metrics\end{tabular}} & \multicolumn{2}{c}{Techniques}         & \multirow{2}{*}{\begin{tabular}[c]{@{}c@{}}Fairness-Privacy-Utility\\Tradeoffs analysis in FL\end{tabular}}  \\ 
\cline{4-5}
                            &                                                                            &                                                                             & Privacy         & Fairness             &                                                                                                              \\ 
\cline{1-5}
\cite{sun2023toward}                      & $(\epsilon, \delta)$-DP                                                    & EO  DemP                                                                    & Gaussian Noise  & Fairness constraints & \ding{55}                                                                                                           \\
\cite{gu2022privacy}                      & /                                                                          & EO  DemP                                                                    & Gaussian Noise  & Fairness constraints & \ding{55}                                                                                                           \\
\cite{corbucci2024puffle}                        & /                                                                          & DemP  DemD                                                                  & DP-SGD          & Fairness constraints & \ding{55}                                                                                                           \\
\textbf{Our Method}         & $(\epsilon, \delta)$-DP                                                    & EO  DemP                                                                    & Noisy attribute & Fairness constraints & $\checkmark$                                                                                                 \\
\bottomrule
\end{tabular}}
\centering
        \begin{tablenotes} 
		\item Notes: EO represents Equalized Odds. DemP represents Demographic Parity. DemD represents Demographic Disparity. 
     \end{tablenotes} 
\end{table*}

\section{Related Work}\label{se: Related Work}
We review three lines of work that intersect with ours: fairness in FL, privacy in FL, and their joint tradeoffs. Table \ref{ta: comparison of fairness and privacy} summarizes the closest prior work.

\subsection{Fairness and Privacy in FL}
Fairness in FL falls into two categories: \textit{client fairness} \cite{li2019fair, martinez2020minimax, yu2020fairness}, which balances utility across participants, and \textit{algorithmic fairness} \cite{hardt2016equality, kusner2017counterfactual}, which mitigates demographic bias. While algorithmic fairness has been extensively studied in centralized settings via debiasing \cite{kairouz2021advances}, it remains underexplored in FL because the server lacks access to local data, making global distribution estimation difficult \cite{mcmahan2017communication}. Existing FL work mostly addresses client fairness through data augmentation \cite{hao2021towards} or distribution balancing \cite{duan2020self}, leaving algorithmic fairness---particularly fairness overfitting \cite{laakom2025fairness}---largely open.

Privacy protection in FL follows two directions. \textit{Cryptographic methods} offer strong guarantees but incur prohibitive overhead \cite{xu2021privacy, he2025pp}, limiting practical use. \textit{Perturbation-based methods}, dominated by differential privacy, are more scalable but introduce a privacy-utility tension as noise on model updates degrades performance \cite{shao2023survey, zhang2025fedlth, li2025clients, he2025fedaa, shen2023share}. Crucially, existing privacy research rarely analyzes how privacy noise affects the system's ability to detect and correct demographic biases, which is a prerequisite for algorithmic fairness.

\subsection{Tradeoffs among Fairness, Privacy, and Utility}
Joint analysis of privacy, fairness, and utility has gained attention in centralized settings, where DP mechanisms can cause inconsistent accuracy drops in both classification \cite{farrand2020neither} and generation \cite{ganev2022robin}, often forcing tradeoffs between objectives \cite{bagdasaryan2019differential, esipova2022disparate, berk2021fairness}. In federated settings, however, such analysis remains limited: existing work \cite{gu2022privacy, sun2023toward} mainly perturbs shared model parameters and often treats privacy and fairness as loosely coupled objectives, rather than analyzing their interaction under sensitive-attribute protection. Our work fills this gap with a game-theoretic framework and rigorous analysis of the privacy-fairness-utility relationship in FL.

\textit{Differences from Jagielski et al. \cite{jagielski2019differentially}.} Our work is complementary but technically distinct: (i) \cite{jagielski2019differentially} targets centralized learning, whereas we address federated settings with client heterogeneity and non-i.i.d. drift (Theorem \ref{th: non-iid tradeoffs}); (ii) we additionally characterize a non-monotonic fairness-utility relationship and its regularization effect; (iii) we protect sensitive attributes via attribute-level DP rather than perturbing model parameters; and (iv) we validate the learner-auditor game on a resource-constrained FL testbed.

\section{Problem Setting and Preliminaries}\label{se: Problem Setting}
To investigate the three critical research questions posed in Sec. \ref{se: Introduction}, we establish a formal framework that enables rigorous analysis of privacy-fairness-utility interactions in FL. Specifically, our framework must support: (1) mathematical analysis of privacy-fairness compatibility (\textbf{Q1}), (2) principled multi-objective algorithm design (\textbf{Q2}), and (3) theoretical characterization of fairness-utility dynamics (\textbf{Q3}). This section introduces our system model (Sec. \ref{se: Federated Learning}), defines fairness metrics that enable game-theoretic optimization (Sec. \ref{se: Fairness Metrices}), and presents privacy mechanisms whose interaction with fairness constraints of our theoretical analysis (Sec. \ref{se: Differential Privacy}).


\subsection{Federated Learning System Model}
\label{se: Federated Learning}

To address \textbf{Q1} about privacy-fairness compatibility, we require a system model that explicitly separates sensitive attributes from other features, enabling us to analyze how privacy protection of demographic information affects fairness detection capabilities.

\textit{Data Distribution and Sensitive Attributes.} Each client $i \in \{1, 2, \ldots, N\}$ holds a local dataset $D_i = \{(x_{ij}, a_{ij}, y_{ij})\}_{j=1}^{m_i}$, where:
\begin{itemize}
    \item $x_{ij} \in \mathcal{X}_i$ represents non-sensitive features (e.g., age, income).
    \item $a_{ij} \in \mathcal{A}_i$ denotes sensitive demographic attributes (e.g., race, gender) requiring privacy protection.
    \item $y_{ij} \in \{0, 1\}$ is the binary prediction target.
\end{itemize}

This explicit separation is crucial for answering \textbf{Q1}, as it allows us to analyze precisely how privacy mechanisms targeting $a_{ij}$ interfere with fairness algorithms that require demographic pattern recognition involving these same attributes.

\textit{Multi-Objective Federated Optimization.} Traditional FL minimizes empirical risk across clients:
\begin{equation}
\boldsymbol{\theta}^* = \arg\min_{\boldsymbol{\theta} \in \boldsymbol{\Theta}} \frac{1}{\mathcal{N}} \sum_{i=1}^{\mathcal{N}} \mathcal{L}(D_i, \boldsymbol{\theta}).
\end{equation}

However, addressing \textbf{Q2} and \textbf{Q3} requires extending this formulation to simultaneously optimize utility, enforce fairness constraints, and maintain privacy protection. The fundamental challenge is that these objectives may be mathematically incompatible, necessitating the game-theoretic framework we develop in Sec. \ref{Se: Fairness-Utility Tradeoff} to navigate their competing demands.

\subsection{Fairness Metrics for Multi-Objective Analysis}
\label{se: Fairness Metrices}

To investigate \textbf{Q3} about fairness-utility dynamics and enable the principled algorithm design required by \textbf{Q2}, we need fairness metrics that are both theoretically analyzable and practically optimizable under privacy constraints.

\textit{Standard Fairness Criteria.} We build upon two fundamental fairness notions that capture different aspects of demographic equity:

\begin{definition}[Demographic Parity (DemP)]
A classifier $f$ satisfies demographic parity with respect to sensitive attribute $A$ if predictions are independent of group membership, i.e. for any $a \in \mathcal{A}$ and $p \in \{0,1\}$, we have
\begin{equation}
P[f(X) = p | A = a] = P[f(X) = p].
\end{equation}
\end{definition}

\begin{definition}[Equalized Odds (EO)]
A classifier $f$ satisfies equalized odds if prediction accuracy is consistent across groups, i.e. for any $a \in \mathcal{A}$, $p$ and $y \in \{0,1\}$, we have
\begin{equation}
P[f(X) = p | A = a, Y = y] = P[f(X) = p | Y = y].
\end{equation}
\end{definition}

\textit{Parameterized Fairness Framework for Theoretical Analysis.} We introduce a unified fairness constraint that allows us to quantify the relationship between fairness enforcement levels and both privacy protection and model utility:

\begin{definition}[$\varepsilon_f$-Fair Classifier]
\label{de: epsilon_f fair classifier}
A classifier $f$ is $\varepsilon_f$-fair with respect to sensitive attribute $A$ if
\begin{equation}
\mathcal{G}_{(y,a)} := \max_{y \in \{0,1\}, a,a' \in \mathcal{A}} |\gamma_{y,a}(f) - \gamma_{y,a'}(f)| \leq \varepsilon_f,
\label{eq: fairness constraint}
\end{equation}
where $\gamma_{y,a}(f) = \mathbb{E}[f(X) | A = a]$ for DemP constraints, and $\gamma_{y,a}(f) = \mathbb{E}[f(X) | A = a, Y = y]$ for EO constraints.
\end{definition}

The parameter $\varepsilon_f$ is central to addressing \textbf{Q3}: by varying $\varepsilon_f$ from strict fairness (small values) to relaxed fairness (large values), we can theoretically characterize the potentially non-monotonic fairness-utility relationship that challenges conventional assumptions about fairness always degrading performance.

\subsection{Differential Privacy and Fairness Detection Interference}\label{se: Differential Privacy}

To answer \textbf{Q1} about privacy-fairness compatibility, we focus on privacy mechanisms that directly interfere with the demographic pattern detection required for fairness enforcement, creating the fundamental mathematical tension our work analyzes.

\textit{Attribute-Level Privacy Model.} Rather than protecting entire model updates, we consider privacy mechanisms that specifically target sensitive attributes, enabling precise analysis of how privacy noise affects fairness detection.

\begin{definition}[$\varepsilon_p$-DP for Sensitive Attributes]
A randomized mechanism $M$ provides $(\varepsilon_p, \delta)$-differential privacy for sensitive attributes if for all outputs $O$ and neighboring datasets $D_S, D_S'$ differing only in sensitive attributes, we have
\begin{equation}
\Pr[M(D_I, D_S) \in O] \leq e^{\varepsilon_p} \Pr[M(D_I, D_S') \in O] + \delta,
\end{equation}
where $D_I$ contains non-sensitive features and $D_S$ contains sensitive attributes.
\end{definition}

\textit{Exponential Mechanism for Controlled Attribute Perturbation.} To enable theoretical analysis of privacy-fairness interactions, we employ the exponential mechanism, which provides quantifiable noise characteristics essential for our mathematical results:

\begin{definition}[Exponential Mechanism]
    \label{de: exponential mechanism}
For sensitive attribute values $a \in \mathcal{A}$, the exponential mechanism outputs:
\begin{equation}
M(s|a) = \begin{cases}
\frac{e^{\varepsilon_p}}{|\mathcal{A}|-1+e^{\varepsilon_p}} & \text{if } s = a \\
\frac{1}{|\mathcal{A}|-1+e^{\varepsilon_p}} & \text{if } s \neq a
\end{cases}
\end{equation}
\end{definition}

This mechanism is crucial for answering \textbf{Q1} because it allows us to mathematically quantify how privacy parameter $\varepsilon_p$ affects the noise level in demographic observations, directly impacting the system's ability to detect fairness violations $\mathcal{G}_{(y,a)} > \varepsilon_f$.

\textit{The Core Privacy-Fairness Detection Problem.} The mathematical heart of \textbf{Q1} lies in this fundamental tension: fairness constraints require detecting statistical differences between groups (i.e., when $|\gamma_{y,a}(f) - \gamma_{y,a'}(f)| > \varepsilon_f$), but privacy mechanisms add noise that makes these differences harder to distinguish from random variation. Specifically, when the exponential mechanism perturbs sensitive attributes with probability $\bar{\pi} = \frac{1}{|\mathcal{A}|-1+e^{\varepsilon_p}}$, the observed demographic statistics become:
\begin{equation}
\tilde{\gamma}_{y,a}(f) = \pi \cdot \gamma_{y,a}(f) + \bar{\pi} \cdot \sum_{a' \neq a} \gamma_{y,a'}(f).
\label{eq: perturbed demographic statistics}
\end{equation}

This creates a direct mathematical conflict: as $\varepsilon_p$ decreases (stronger privacy), $\bar{\pi}$ increases, making true demographic differences $|\gamma_{y,a}(f) - \gamma_{y,a'}(f)|$ increasingly difficult to detect above the noise threshold $\varepsilon_f$.

\textit{Framework Summary for Research Questions.} This framework enables us to formalize and answer our three research questions:
\begin{itemize}
    \item \textbf{Q1:} The privacy-fairness compatibility analysis centers on how (\ref{eq: perturbed demographic statistics}) affects our ability to detect violations in (\ref{eq: fairness constraint}).
    \item \textbf{Q2:} The competing objectives (utility loss, fairness constraint $\varepsilon_f$, privacy budget $\varepsilon_p$) require the game-theoretic optimization framework developed in Sec. \ref{Se: Fairness-Utility Tradeoff}.
    \item \textbf{Q3:} The fairness-utility dynamics can be analyzed by varying $\varepsilon_f$ in (\ref{eq: fairness constraint}) and measuring its impact on generalization performance.
\end{itemize}

Sec. \ref{se: Optimization Problem in FL} leverages this framework to develop our game-theoretic algorithm and provide theoretical guarantees addressing all three research questions.

\section{Problem Formulation and Algorithm Design}\label{Se: Fairness-Utility Tradeoff}


This section develops our game-theoretic framework for privacy-fairness-utility optimization in FL. We first formulate the multi-objective optimization problem using Lagrangian duality (Sec. \ref{se: Optimization Problem in FL}), addressing \textbf{Q1} about privacy-fairness compatibility. We then present the \textit{FedPF} algorithm that balances these competing objectives (Sec. \ref{se: Private and Fair Federated Learning Algorithm}), tackling \textbf{Q2}'s optimization challenge. Finally, we provide theoretical analysis establishing convergence guarantees and tradeoff relationships (Sec. \ref{se: Error Analysis of FedPF} and \ref{se: Convergence and Robustness}), resolving \textbf{Q3} about fairness-utility dynamics.

\subsection{Optimization Problem in FL}\label{se: Optimization Problem in FL}
Fairness in FL requires preventing classifiers from discriminating against specific demographic groups (e.g., based on gender or race) by ensuring consistent prediction behaviors across these groups. In an FL system with \(\mathcal{N} = \{1, 2, \dots, N\}\) clients, each client \(i \in \mathcal{N}\) holds a local dataset used to train its local model, while preserving data privacy throughout the collaborative training process.

\subsubsection{Client-Side Optimization} 
Building on the framework established in Sec. \ref{se: Federated Learning}, each client $i$ holds a local dataset $\{(x_{ij}, a_{ij}, y_{ij})\}_{j=1}^{m_i}$ with the explicit separation of sensitive attributes $a_{ij}$ that enables our privacy-fairness analysis.

Let \(f_i(\cdot; \boldsymbol{\theta}_i): \mathcal{X}_i \to \{0, 1\}\) denote the local classifier of client $i$, parameterized by \(\boldsymbol{\theta}_i \in \boldsymbol{\Theta}_i\). The client aims to train \(f_i\) to minimize empirical prediction loss while satisfying fairness constraints. We formalize this as a constrained optimization problem for client $i$:
\begin{equation}
    \begin{aligned} 
    \mathcal{L}\left( f_i; \boldsymbol{\theta}_i, \varepsilon_{f_i} \right) = & \min_{\boldsymbol{\theta}_i \in \boldsymbol{\Theta}_i} \, \sum_{j=1}^{m_i}{err}(f_i) \\ \text{s.t.} \quad & \mathcal{G}_{i, (ya)} \leq \varepsilon_{f_i},  
\end{aligned}
\label{eq: client_fair_optim} 
\end{equation}
where ${err}(f_i)= \ell\left( f_i(x_{ij}; \boldsymbol{\theta}_i), y_{ij} \right)$ is the empirical loss function (e.g., cross-entropy loss) for client $i$, and \(\varepsilon_{f_i} > 0\) is the fairness tolerance threshold for criterion at client $i$. When $\varepsilon_f$ is \textit{EO} constraint, \(\mathcal{G}_{i, (ya)} = \max_{y \in \{0, 1\}} \max_{a, a' \in \mathcal{A}_i} \left| \gamma_{y, a}(f_i) - \gamma_{y, a'} (f_i) \right|\) quantifies fairness discrimination, and \(\gamma_{y, a}(f_i) = \mathbb{P}\left( f_i(x; \boldsymbol{\theta}_i) = 1 \mid Y = y, A = a \right)\) is the conditional probability of predicting 1 given label $y$ and sensitive attribute $a$. 

To solve the server optimization problem in (\ref{eq: client_fair_optim}), we consider \textit{Lagrangian} relaxation. By introducing dual variables \(\lambda_i \geq 0\) for the fairness constraints, the problem is transformed into a min-max optimization problem. Inspired by fairness reduction frameworks \cite{agarwal2018reductions}, the \textit{Lagrangian} form of the client-side objective becomes:
\begin{equation}
    \begin{aligned} 
        \min_{f_i \in \mathcal{F}_i} \max_{\lambda_i \in \Lambda_i} \mathcal{L}(f_i, \lambda_i) := & \,\sum_{j=1}^{m_i} {err}(f_i) + \lambda_i  \left( \mathcal{G}_{i, (ya)} - \varepsilon_{f_i} \right), \label{eq: client_lagrangian} 
    \end{aligned}
\end{equation}
where \(\mathcal{F}_i\) is the client $i$ hypothesis space, which represents the set of all possible functions or models that client \(i\) uses to train its local model during the FL process. \(\Lambda_i = \{ \lambda_i \mid \lambda_i \geq 0, \|\lambda_i\|_1 \leq B \}\) is the feasible set of dual variables (bounded by \(B > 0\) to ensure convergence), and \(\lambda_i (\cdot)\) denotes the weighted fairness penalty.

\subsubsection{Server-Side Aggregation}
The server coordinates global optimization by aggregating local objectives while maintaining the privacy-fairness balance established at the client level. The server's goal is to align local fairness constraints and minimize the global empirical loss, ensuring consistent performance across demographic groups globally.
\begin{equation}
    \begin{aligned} \min_{f_i \in \mathcal{F}} \max_{\boldsymbol{\lambda}_i \in \Lambda} \, \frac{1}{\mathcal{N}} \sum_{i=1}^{\mathcal{N}} \mathcal{L}(f_i, \lambda_i), \label{eq: optimization_problem} 
    \end{aligned}
\end{equation}
where \(f_i \in \mathcal{F}\) represents the global classifier (aggregated from local models \(\{f_i\}_{i=1}^{\mathcal{N}}\)), \(\boldsymbol{\lambda} = (\lambda_1, \lambda_2, \dots, \lambda_N) \in \Lambda = \prod_{i=1}^N \Lambda_i\) is the vector of global dual variables, the objective aggregates client-side \textit{Lagrangian} functions (\ref{eq: client_lagrangian}) to enforce global fairness.

For the optimization problems (\ref{eq: optimization_problem}), strong duality is established under the convex surrogate used by reduction-based fair classification \cite{agarwal2018reductions}. Specifically, we optimize a convex surrogate of the fairness constraint and then solve the resulting saddle-point problem.
\begin{equation}
    \begin{aligned} \min_{f_i \in \mathcal{F}} \max_{\boldsymbol{\lambda}_i \in \Lambda} \frac{1}{\mathcal{N}} \sum_{i=1}^{\mathcal{N}} \mathcal{L}(f_i, \lambda_i) = \max_{\boldsymbol{\lambda}_i \in \Lambda} \min_{f_i \in \mathcal{F}} 
    \frac{1}{\mathcal{N}} \sum_{i=1}^{\mathcal{N}} \mathcal{L}(f_i, \lambda_i). \label{eq: strong_duality} \end{aligned}
\end{equation}
\begin{proposition}[Convexity and Saddle-Point Existence]
Under linear (or convex-score) hypotheses with convex loss $\ell$, each fairness moment used in the reduction is affine in the randomized predictor, and $|\cdot|$ constraints are represented by two linear inequalities over moments. Therefore, the feasible set is convex in the primal variable and compact in the dual variable $\Lambda_i=\{\lambda_i\ge 0:\|\lambda_i\|_1\le B\}$. By standard minimax arguments (Sion's theorem), a saddle point exists and Eq. (\ref{eq: strong_duality}) holds for the surrogate problem \cite{agarwal2018reductions}.
\end{proposition}

This duality enables a two-stage solution via a zero-sum game between two players: 
\begin{itemize}
    \item \textbf{\textit{Learner (client-side)}}: Minimizes the aggregated \textit{Lagrangian} by updating local models \(\{f_i\}\) to reduce empirical loss while respecting local fairness constraints.
    \item \textbf{\textit{Auditor (server-side)}}: Maximizes the aggregated \textit{Lagrangian} by adjusting dual variables \(\{\lambda_i\}\) to identify and penalize fairness violations, ensuring global fairness alignment.
\end{itemize}

By iterating between local updates (client-side) and global aggregation (server-side), the framework converges to an optimal solution \((f^*, \boldsymbol{\lambda}^*)\), where \(f^*\) is a globally fair classifier and \(\boldsymbol{\lambda}^*\) are optimal dual variables balancing loss and fairness.

\subsection{Private and Fair FL Algorithm}\label{se: Private and Fair Federated Learning Algorithm}
To solve the min-max problem in (\ref{eq: strong_duality}), we design a private and fair FL algorithm \ref{al: fedpf}. The core idea is to leverage the reduction method from Agarwal et al. \cite{agarwal2018reductions}, where the learner’s best response to a given $\lambda$ (denoted $\text{BEST\_F}(\lambda)$ in Algorithm \ref{al: BEST}) reduces to a cost-sensitive classification problem. This algorithm integrates DP mechanisms to protect sensitive data while enforcing fairness constraints during local training and global aggregation.

\subsubsection{Algorithm Design Rationale}
\textit{FedPF} addresses three key challenges in FL: ensuring fairness across clients with diverse sensitive attributes while minimizing utility loss, protecting sensitive attributes from privacy leaks during model uploads via differential privacy, and balancing utility loss from privacy and fairness constraints through adaptive parameter updates. The algorithm operates in a multi-round framework where clients and the server iteratively communicate to refine the global model. Each client performs local training with fairness constraints and privacy protection, then uploads updated parameters to the server for aggregation. The server aggregates local parameters to form a global model, which is then broadcast back to clients for the next round.

\subsubsection{Core Formulations of $\text{BEST\_F}(\lambda)$} Next we rewrite  the \textit{Lagrangian} function in (\ref{eq: optimization_problem}) for cost-sensitive learning in Algorithm \ref{al: BEST}.

According to the fair learning framework of Agarwal \textit{et al.} \cite{agarwal2018reductions}, when the dual variable $\lambda$ is fixed, the learner's optimal response can be solved through cost-sensitive classification. In this paper, we take binary classification as an example and translate (\ref{eq: optimization_problem}) to a cost-sensitive problem on $\left\{\left(x_j, c_j^0, c_j^1\right)\right\}_{j=1}^m$ with costs (\textit{EO} constraint) and obtain:
\begin{equation}
    f^*=\underset{f \in \mathcal{F}}{\arg \min } \sum_{j=1}^m \{f\left(x_j\right) c_j^1+\left(1-f\left(x_j\right)\right) c_j^0 \},
\end{equation}
\begin{equation}
        \begin{aligned}
        c_j^0 &\leftarrow \mathbf{1}\{y_j \ne 0\}, 
        c_j^1 &\leftarrow \mathbf{1}\{y_j \ne 1\} + \frac{\lambda_{(\tilde{a}_j, y_j)} - \mu_{y_j}}{p_{\tilde{a}_j, y_j}},
    \end{aligned}
\end{equation}
where $c_j^0$ and $c_j^1$ are the cost terms for misclassifying labels 0 and 1, respectively. The costs are adjusted based on the dual parameters $\lambda$ for fairness constraints, where $\mu_{y_i}$ represents the mean values of the dual variables, and $p_{\tilde{a}_j,y_j}$ represents the empirical probabilities of group $(\tilde{a}_j, y_j)$ for \textit{EO} constraint. This form comes from rewriting the fairness term in the Lagrangian into per-sample costs: $(\lambda_{(\tilde{a}_j,y_j)}-\mu_{y_j})$ captures each group-specific penalty relative to the label-wise average dual weight, while division by $p_{\tilde{a}_j,y_j}$ yields an unbiased empirical scaling under the reduction framework \cite{agarwal2018reductions}. If the constraint is \textit{DemP} constraint, the cost terms are set as follows:
\begin{equation}
    c_j^0 = \mathbf{1}\{y_j \ne 0\}, \quad c_j^1 = \mathbf{1}\{y_j \ne 1\} + \frac{\lambda_{\tilde{a}_j} - \mu_{y_j}}{p_{\tilde{a}_j}}.
\end{equation}

\begin{algorithm}[t]
    \small
\caption{\textsc{FedPF: Private and Fair Federated Learning}}
\label{al: fedpf}
\begin{algorithmic}[1]
\REQUIRE $\mathcal{D}_i = \{(x_j, a_j, y_j)\}_{j=1}^{m_i}$,  $\eta_\theta$, $\eta_\lambda$, $\varepsilon_p$,$\varepsilon_f$, $B$, $T$
\ENSURE Final global model parameters $\theta^G$
\STATE Server initialization: Global model $\theta_0^G$, global dual variables $\boldsymbol{\lambda}_0 = (\lambda_{1,0}, ..., \lambda_{N,0}) \leftarrow \mathbf{0}$ \\
\FOR{$t = 1$ to $T$}
    \STATE Server broadcast: Send $\theta_{t-1}^G$ and $\boldsymbol{\lambda}_{t-1}$ to all clients $i$
    \FOR{each client $i \in \mathcal{N}$ \textbf{in parallel}}
        \STATE Local model initialization: $\theta_{i,0} \leftarrow \theta_{t-1}^G$
        \FOR{mini-batch $b \subset \mathcal{D}_i$}
            \STATE \textit{\color{olive!100} //Privacy protection: perturb sensitive attributes}
            \STATE Replace $a_j$ with $\tilde{a}_j$ using randomized response with budget $\varepsilon_p$, obtaining $\tilde{b} = \{(x_j, \tilde{a}_j, y_j)\}$
            \STATE \textit{\color{olive!100}// Cost-sensitive classification under fixed global dual variables}
            \STATE \textbf{\color{magenta!80}{Learner}} update: $f_i^{(t)} \leftarrow \textbf{\textsc{Best\_f}}(\lambda_i, \tilde{b})$ 
            \STATE Compute subgradient: $\partial_{\theta_i} \mathcal{L}(f_i; \theta_i, \lambda_i) = \nabla \sum err(f_i) + \lambda_i \cdot \partial \mathcal{G}_{i,(ya)}$
            \STATE Update local model: $\theta_{i,t+1} \leftarrow \theta_{i,t} - \eta_\theta \cdot \partial_{\theta_i} \mathcal{L}(f_i; \theta_i, \lambda_i)$
        \ENDFOR
        \STATE Upload local model: Send $\theta_{i,t+1}$ and $\lambda_i$ to the server
    \ENDFOR
    \STATE Server updates: $\theta_t^G \leftarrow \frac{1}{\mathcal{N}} \sum_{i=1}^{\mathcal{N}} \theta_{i,t+1}$
        \STATE \textbf{\color{magenta!80}Auditor} update: $\boldsymbol{\lambda}_t \leftarrow \boldsymbol{\lambda}_{t-1} + \eta_\lambda \cdot \nabla_{\boldsymbol{\lambda}} \sum \mathcal{L}$
\ENDFOR
\end{algorithmic}
\end{algorithm}
In implementation, $|\gamma_{y,a}-\gamma_{y,a'}|$ is handled via subgradients (equivalently, by splitting it into two linear constraints $\gamma_{y,a}-\gamma_{y,a'}\le \varepsilon_f$ and $\gamma_{y,a'}-\gamma_{y,a}\le \varepsilon_f$), which avoids nondifferentiability at zero and is consistent with reduction-based fair optimization \cite{agarwal2018reductions}.

\begin{algorithm}[ht]
    \small
\caption{\textsc{Best\_f: Cost-Sensitive Classifier}}
\label{al: BEST}
\begin{algorithmic}[1]
\REQUIRE Global dual variable component $\lambda_i$, privacy-processed mini-batch $\tilde{b} = \{(x_j, \tilde{a}_j, y_j)\}$
\ENSURE Optimal classifier $f^*$
\FOR{$j=1$ to $|\tilde{b}|$}
    \STATE Compute empirical probabilities: $p_{\tilde{a}_j, y_j} = \frac{|\{k: \tilde{a}_k = \tilde{a}_j, y_k = y_j\}|}{|\tilde{b}|}$
    \STATE Retrieve dual means: $\mu_{y_j}$
    \STATE Revised cost terms:
    \[\small
    \begin{aligned}
        c_j^0 &\leftarrow \mathbf{1}\{y_j \ne 0\}  \\
        c_j^1 &\leftarrow \mathbf{1}\{y_j \ne 1\} +  \frac{\lambda_i(\tilde{a}_j, y_j) - \mu_{y_j}}{p_{\tilde{a}_j, y_j}} (\textit{EO constraint}) \\
    \end{aligned}
    \]
\ENDFOR
\STATE Solve: {\small$f^* = \arg\min_{f \in \mathcal{F}} \sum_{j=1}^{|\tilde{b}|} \left[f(x_j) \cdot c_j^1 + (1 - f(x_j)) \cdot c_j^0\right]$}\\
\RETURN $f^*$
\end{algorithmic}
\end{algorithm}

\subsection{Error Analysis of \textit{FedPF} Algorithm} \label{se: Error Analysis of FedPF}

In this section, we analyze the relationship between fairness, privacy, and utility of the proposed \textit{FedPF} algorithm. We first formalize key assumptions and definitions, then present and interpret the unified theorem characterizing the algorithm's performance. To establish theoretical guarantees, we adopt the following assumptions, consistent with the foundational frameworks \cite{agarwal2018reductions,jagielski2019differentially,mozannar2020fair}:

\begin{itemize}
    \item For the true demographic disparity $\Delta = |\gamma_{y,a}(f) - \gamma_{y,a'}(f)|$ and the fairness threshold $\varepsilon_f$, if the privacy parameter $\varepsilon_p$ satisfies:$\varepsilon_p < \ln\left( \frac{|A| - 1}{\frac{\Delta}{\varepsilon_f} - 1} \right)$, then the system cannot reliably detect fairness violations, i.e., privacy protection will cover up discrimination.
    \item \textit{Model Complexity}: The Rademacher complexity of the classifier family \(\mathcal{F}\) is bounded, i.e., \(\mathfrak{R}_m(\mathcal{F}) \leq C m^{-\alpha}\), where \(C\) is a constant and \(\alpha \leq 1/2\) \cite{agarwal2018reductions}.
    \item \textit{Dual Variable Constraints}: The dual variable \(\lambda\) (used to enforce fairness constraints) has a bounded \(\ell_1\)-norm: \(\|\lambda\|_1 \leq B\).
\end{itemize}

We aim to bound the error gap between $\hat{Y}$ and $Y^*$, where $Y^*$ is the best classifier that satisfies the fairness constraints $\varepsilon_f$ and privacy budget $\varepsilon_p$. The following theorem provides the error bound of the \textit{FedPF} algorithm. For the \textit{iid} scenarios, we have the following result:

\begin{theorem}[Privacy-Fairness-Utility Tradeoff of \textit{FedPF}]
    Let \(\hat{Y}\) be the output of the \textit{FedPF} algorithm, and \(Y^*\) be the optimal classifier satisfying \((\varepsilon_p, \delta)\)-differential privacy (DP) and \(\varepsilon_f\)-fairness constraints (EO, DemP). If client data are independently and identically distributed and the Rademacher complexity of the classifier family is bounded, with probability at least \(1-\beta\), the following inequality holds:
    \begin{equation}
        \footnotesize
        \begin{aligned}
            \mathrm{err}(\hat{Y}) \leq \mathrm{err}(Y^*) &+ O \left( \frac{B^2 \varepsilon_f^4 T^{3/2} \mathcal{H}}{\varepsilon_p^2} \right) + O\left( \mathfrak{R}_m(\mathcal{F}) + \sqrt{\frac{\log(1/\beta)}{m}}\right),\\
            \mathrm{err}(\hat{Y}) \geq \mathrm{err}(Y^*) &+ \Omega\left( \frac{|\mathcal{A}| \log(1/\varepsilon_f)}{\sqrt{m}} - \mathfrak{R}_m(\mathcal{F}) \right) + \Omega\left( \sqrt{\frac{\log(1/\beta)}{m}} \right),
        \end{aligned}
    \end{equation}
    where \(\mathcal{H} = \ln(1/\delta) \ln^2(8T|\mathcal{A}|/\delta) \log(|\mathcal{K}| + 1)\). $\mathfrak{R}_m(\mathcal{F})$ is Rademacher complexity of the
classifier family $\mathcal{F}$, \(\mathfrak{R}_m(\mathcal{F}) \leq C m^{-\alpha}\), \(\alpha \leq 1/2\). \(B\) is the \(\ell_1\)-norm bound of dual variables \(\lambda\), \(T\) is the number of training rounds. \(m\) is the client-side sample size. \(|\mathcal{A}|\) is the number of sensitive attribute categories. $\mathcal{K}$ is a set of indices, and each index $k \in \mathcal{K}$ corresponds to a fairness constraint. If fairness constraints is \textit{DemP}, then $|\mathcal{K}| = 2|\mathcal{A}|$, if fairness constraints is \textit{EO}, then $|\mathcal{K}| = 4|\mathcal{A}|$. $\delta$ is the parameter of DP mechanism.
    \label{th: tradeoffs}
\end{theorem}
\textbf{Proof sketches are summarized in Appendix, including the convergence argument for the learner-auditor updates.}
\begin{myremark}{Remark 1: Privacy-Fairness Coupling}
\small
Theorem \ref{th: tradeoffs} provides a rigorous upper bound on the excess risk of \textit{FedPF}, decomposing the error into learning regret, auditing regret, and generalization error. Crucially, the term involving $O(1/\varepsilon_p^2)$ reveals a fundamental coupling: as the privacy budget $\varepsilon_p$ decreases (stronger privacy), the noise introduced by the exponential mechanism increases the variance of the auditor's estimates. This reduces the statistical power to detect fairness violations, theoretically confirming that strict privacy protection can "mask" underlying demographic biases.
\end{myremark}

\begin{theorem}[Extended for Non-I.I.D. of Theorem 1]
\label{th: non-iid tradeoffs}
Let $\hat{Y}$ be the output of the \textit{FedPF} algorithm, and $Y^*$ be the optimal classifier satisfying $(\varepsilon_p, \delta)$-differential privacy (DP) and $\varepsilon_f$-fairness constraints (EO, DemP). Assume non-i.i.d. data with bounded heterogeneity $\Gamma = \max_{i,j} \|\nabla \mathcal{L}(D_i) - \nabla \mathcal{L}(D_j)\|^2$ (gradient dissimilarity across clients) and bounded local variance $\sigma^2 = \mathbb{E}_i[\text{Var}(\nabla \mathcal{L}(D_i))]$. If the Rademacher complexity of the classifier family is bounded, with probability at least $1 - \beta$, the following inequalities hold:
\begin{equation}
\footnotesize
\begin{aligned}
&\operatorname{err}(\hat{Y}) \leq \operatorname{err}(Y^*) + O\left(\mathcal{A}\right) + O\left(\mathcal{B} + \frac{\Gamma T}{N} + \sigma^2 \sqrt{\frac{T}{mN}}\right), \\
&\operatorname{err}(\hat{Y}) \geq \operatorname{err}(Y^*) + \Omega\left(\mathcal{C} + \frac{\Gamma T}{N} + \sigma^2 \sqrt{\frac{T}{mN}}\right) + \Omega\left(\sqrt{\frac{\log(1/\beta)}{m}} \right),
\end{aligned}
\end{equation}
where all terms match Theorem 1, $\mathcal{A} = \frac{B^2 \varepsilon_f^4 T^{3/2} \mathcal{H}}{\varepsilon_p^2}$, $\mathcal{B} =  \mathfrak{R}_m(\mathcal{F}) + \sqrt{\frac{\log(1/\beta)}{m}}$, $\mathcal{C}=\frac{|\mathcal{A}| \log(1/\varepsilon_f)}{\sqrt{m}} - \mathfrak{R}_m(\mathcal{F})$, plus non-i.i.d. penalties: $\frac{\Gamma T}{N}$ (heterogeneity drift over $T$ rounds, averaged over $N$ clients) and $\sigma^2 \sqrt{\frac{T}{mN}}$ (local variance amplified by heterogeneity). This extends Theorem 1 by incorporating data heterogeneity, showing that non-i.i.d. distributions inflate both upper and lower bounds on error, exacerbating privacy-fairness-utility tradeoffs.

\end{theorem}

\begin{myremark}{Remark 2: Impact of Data Heterogeneity}
\small
Theorem \ref{th: non-iid tradeoffs} extends the convergence analysis to the distributed non-IID setting, a core challenge in practical federated systems. By introducing the gradient dissimilarity parameter $\Gamma$ and local variance $\sigma^2$, the bound quantifies how data heterogeneity penalizes model utility. The analysis shows that the error bound scales with the degree of distribution shift across clients, but the impact is mitigated by the number of participants $N$. This implies that while \textit{FedPF} is robust to non-IID data, the convergence rate is sensitive to the weight drift caused by extreme local biases.
\end{myremark}

\subsection{Convergence and Robustness}\label{se: Convergence and Robustness}
We analyze robustness to distribution shifts between sensitive groups ($p$) and protected groups ($\hat{p}$) using the Total Variation (TV) distance, i.e. $TV(p, \hat{p})$, which measures distribution divergence between $p$ and $\hat{p}$.


\begin{theorem}[Fairness Discrimination Bound of \textit{FedPF}]
The total fairness discrimination across all clients is bounded by:

\begin{equation}
    \small
         \sum_{i = 1}^{\mathcal{N}} (\gamma_{y, \hat{p},i} (f_i) - \gamma_{y, p, i}(f_i) ) \leq \mathcal{N} \cdot \alpha_{\max},
    \end{equation}

where $\alpha_{\max} = \max_{i \in \mathcal{N}} \{ TV\left(p_i, \hat{p}_i\right)\}$.
    \label{th: fairness_constraints}
\end{theorem}

\begin{myremark}{Remark 3: Interpretation of Theorem \ref{th: fairness_constraints}}
\small
Theorem \ref{th: fairness_constraints} is a robustness statement: the aggregate fairness-discrimination deviation is upper-bounded by the maximum client-level distribution shift measured by $TV(p_i,\hat p_i)$. Hence, when protected-group distributions drift across clients, the global fairness degradation remains controlled by a deterministic envelope $\mathcal{N}\alpha_{\max}$. This remark clarifies the theorem's scope as a distribution-shift fairness bound rather than a direct claim about non-monotonic utility.
\end{myremark}

\section{Experimental Validation}\label{se: Evalution}
In this section, we provide comprehensive experimental validation of the \textit{FedPF} algorithm's theoretical guarantees. We systematically evaluate the privacy-fairness-utility tradeoffs through controlled experiments, examining how algorithmic parameters affect model performance across different fairness constraints and privacy budgets.
\subsection{Experimental Setup and Methodology}

\subsubsection{Datasets and Problem Formulation}
We evaluate the proposed \textit{FedPF} on three widely-used fairness benchmark datasets, including \textit{Adult}, \textit{Bank}, and \textit{Compas}. 

\begin{itemize}
    \item \textit{Adult} \cite{tran2021differentially}: The dataset includes 45,221 records, containing the personal income of different persons with binary labels. Income over \$ 50K, the label is 1, and the inverse is 0. We chose the \textit{Age} as the sensitive attribute.
    \item \textit{Bank} \cite{tran2021differentially}:
    The dataset contains records of 30,000 customers and their credit card transactions with a bank. The output is the default status (\textit{default\_payment}) (Overdue 1, Not Overdue 0) of the customer for the next month's repayment. We choose the \textit{Age} as the sensitive attribute.
    \item \textit{Compas} \cite{tran2021differentially}: The dataset includes 11,750 criminal records in the US. We set the \textit{score\_text} class as a binary classification problem (0: \textit{Low}, 1: \textit{Medium/High}). We choose the \textit{Sex} as the sensitive attribute.

\end{itemize}

\begin{figure}
    \centering
    \includegraphics[width=0.4\textwidth]{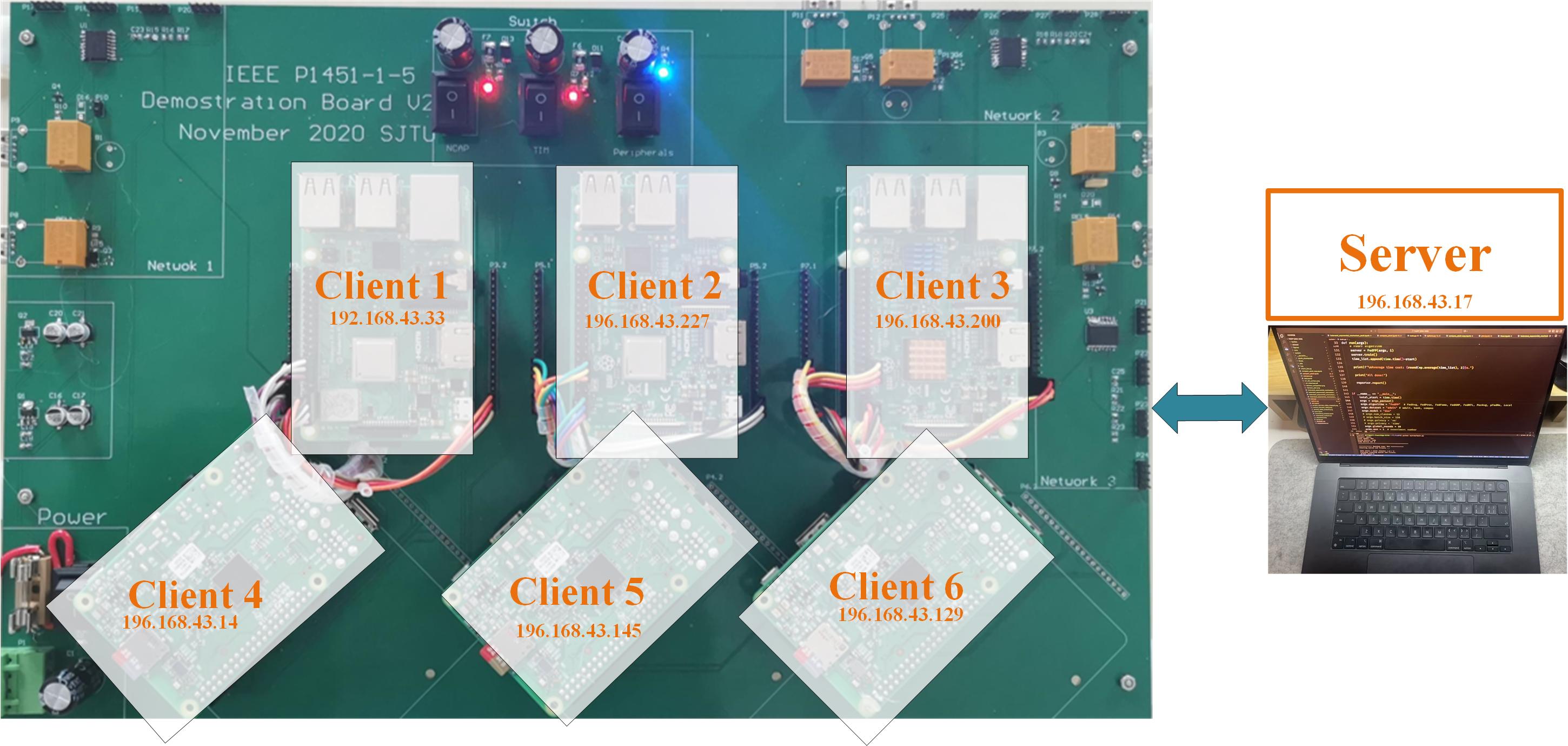}
    \caption{The experimental setup for FL settings, where we consider 6 clients (Raspberry Pi) and a server (Laptop). Each client has a local dataset with sensitive attributes and trains a local model with privacy and fairness constraints. The server aggregates the local models to update the global model.}
    \label{fig: experiment setup}
\end{figure}

\begin{table}
\centering
\caption{End-to-End Training Overhead}
\label{ta: training overhead}
\resizebox{\linewidth}{!}{
\begin{tabular}{ccccc} 
\toprule
\textbf{Method}            & \begin{tabular}[c]{@{}c@{}}\textbf{Avg Round }\\\textbf{Latency (ms)}\end{tabular} & \begin{tabular}[c]{@{}c@{}}\textbf{Client }\\\textbf{Computation (ms)}\end{tabular} & \begin{tabular}[c]{@{}c@{}}\textbf{Server }\\\textbf{Overhead (ms)}\end{tabular} & \begin{tabular}[c]{@{}c@{}}\textbf{Total Training }\\\textbf{Time (s)}\end{tabular}  \\ 
\midrule
Vanilla FL                 & 320                                                                                & 280                                                                                 & 40                                                                               & 32.0                                                                                 \\
+ DP only ($log(\varepsilon_p)=1$) & 355                                                                                & 300                                                                                 & 55                                                                               & 35.5                                                                                 \\
+ Fairness only ($\varepsilon_f=0.1$)           & 370                                                                                & 305                                                                                 & 65                                                                               & 37.0                                                                                 \\
DP + Fairness (\textit{FedPF})      & 410                                                                                & 320                                                                                 & 90                                                                               & 41.0                                                                                 \\
\bottomrule
\end{tabular}}
\label{ta: training overhead}
\end{table}

\subsubsection{Algorithm Baselines} To evaluate the performance and efficiency of the \textit{FedPF} algorithm, we compare four different algorithms, including \textit{FedAvg} \cite{mcmahan2017communication}, \textit{CENTAUR} \cite{shen2023share}, \textit{FedAA} \cite{he2025fedaa} and \textit{FedCEO} \cite{li2025clients}. As far as we know, there is relatively little literature that considers both privacy and fairness. In order to evaluate algorithmic consistency, we have set the above algorithm to consider fairness metrics during local training.

\begin{table}[t]
\centering
\caption{Communication Overhead per Round}
\resizebox{\linewidth}{!}{
\begin{tabular}{lccc}
\toprule
\textbf{Method} & \textbf{Uplink per Client (KB)} & \textbf{Downlink (KB)} & \textbf{Total Comm / Round (KB)} \\
\midrule
Vanilla FL                  & 512 & 512 & 3584 \\
+ DP only ($log(\varepsilon_p)=1$)                  & 512 & 512 & 3584 \\
+ Fairness only ($\varepsilon_f=0.1$)             & 540 & 520 & 3740 \\
DP + Fairness (\textit{FedPF})       & 560 & 540 & 3900 \\
\bottomrule
\end{tabular}}
\label{tab: comm_overhead}
\end{table}

\begin{table*}
\centering
\caption{The performance comparison of fairness and privacy under different privacy budgets $(log(\varepsilon_p))$ on \textit{Adult} dataset.}
\label{ta: comparison of fairness and privacy}
\resizebox{\linewidth}{!}{
\begin{tabular}{ccccccccccccc} 
\toprule
\multirow{2}{*}{\begin{tabular}[c]{@{}c@{}}Privacy\\Budget\end{tabular}} & \multicolumn{5}{c}{Global Mode Accuracy (\%)}                                     &  & \multicolumn{6}{c}{Global Model Discrimination ($\mathcal{G}_{ya}$)}                                                                     \\ 
\cline{2-6}\cline{8-13}
                                                                         & FedAvg \cite{mcmahan2017communication}      & FedAA  \cite{he2025fedaa}     & FedCEO  \cite{li2025clients}    & CENTAUR \cite{shen2023share}    & \textbf{FedPF(Ours)}          &  & FedAvg \cite{mcmahan2017communication}      & FedAA \cite{he2025fedaa}      & FedCEO  \cite{li2025clients}             & CENTAUR  \cite{shen2023share}            & \multicolumn{2}{c}{\textbf{FedPF(Ours)}}           \\ 
\hline
-2                                                                       & 0.840±0.002 & 0.841±0.001 & 0.839±0.001 & 0.839±0.001 & \textbf{0.842}±0.001 &  & 0.853±0.196 & 0.784±0.217 & \textbf{0.656}±0.192 & 0.752±0.194          & \ \ \ \ \underline{0.711}±0.191 (2nd) &                             \\
-1                                                                       & 0.841±0.001 & 0.840±0.001 & 0.839±0.001 & 0.841±0.001 & 0.837±0.003          &  & 0.847±0.198 & 0.723±0.281 & 0.722±0.101          & \textbf{0.667}±0.181 & \multicolumn{2}{c}{\underline{0.715}±0.029 (2nd)}           \\
0                                                                        & 0.841±0.001 & 0.840±0.001 & 0.839±0.001 & 0.841±0.002 & 0.838±0.004          &  & 0.830±0.223 & 0.749±0.250 & 0.807±0.075          & 0.745±0.081          & \multicolumn{2}{c}{\textbf{0.690}±0.229 ($\downarrow$)}   \\
1                                                                        & 0.831±0.015 & 0.831±0.010 & 0.834±0.002 & 0.831±0.012 & \textbf{0.841}±0.019 &  & 0.881±0.178 & 0.783±0.216 & 0.832±0.087          & 0.845±0.232          & \multicolumn{2}{c}{\textbf{0.396}±0.227 ($\downarrow$)} \\
2                                                                        & 0.828±0.011 & 0.831±0.005 & 0.829±0.013 & 0.834±0.006 & \textbf{0.845}±0.014 &  & 0.901±0.143 & 0.827±0.176 & 0.891±0.167          & 0.884±0.171          & \multicolumn{2}{c}{\textbf{0.368}±0.259 ($\downarrow$)}   \\
\bottomrule
\end{tabular}}
\end{table*}

\subsubsection{Implementation Details} In our experiments, we implement the proposed \textit{FedPF} algorithm based on the PyTorch framework. The privacy mechanism used is the Exponential Mechanism (Definition \ref{de: exponential mechanism}) to ensure differential privacy during local model training on each client. We set the privacy budget $\varepsilon_p$ to vary in $\{0.01, 0.1, 1.0, 10.0, 100.0\}$ to evaluate the impact of different privacy levels on model performance. The fairness tolerance parameter $\varepsilon_f$ is set to vary in $\{0.01, 0.1, 1.0\}$ to assess the effect of different fairness constraints. Figure \ref{fig: experiment setup} illustrates our real-world federated learning testbed based on IEEE P1451 protocol \cite{homeb}, which consists of six Raspberry Pi devices acting as clients and a laptop serving as the central aggregator. Each client maintains a local dataset containing sensitive attributes and performs on-device training with privacy and fairness constraints enabled. The server coordinates the training process by aggregating local model updates and broadcasting the updated global model. This setup captures key system characteristics of practical FL deployments, including resource-constrained edge devices, decentralized data ownership, and realistic communication patterns.
\begin{itemize}
    \item \textbf{$\varepsilon_f$ without $\varepsilon_p$ in \textit{FedPF}}: This algorithm is the baseline without privacy protection, to evaluate the improvement of the different group fairness constraints of \textit{FedPF} algorithm.
    \item \textbf{$\varepsilon_p$ without $\varepsilon_f$ in \textit{FedPF}}: This algorithm is the baseline without fairness constraints, to evaluate the improvement of the different privacy protection levels of \textit{FedPF} algorithm.
    \item \textbf{$\varepsilon_p$ and $\varepsilon_f$ in \textit{FedPF}}: This algorithm considers both privacy and fairness constraints our proposed one.
\end{itemize}

In the experiment, we use the \textit{FedAvg} algorithm on the server to aggregate the global model. The number of clients is 6. The local and global rounds are 1 and 5, respectively. We use a simple feedforward neural network consisting of three fully connected layers. The input is first projected to a hidden dimension with a ReLU activation, followed by a second linear layer mapping to 200 dimensions, and finally, an output layer producing a 2-dimensional prediction. The input dimensions of the model are 12, 22, and 12 based on \textit{Adult}, \textit{Bank}, and \textit{Compas} datasets, respectively. The local batch size is 128. Please refer to the github for the detailed experimental settings.

\subsubsection{Evaluation Metrics} To evaluate the influence of fairness and privacy on the FL global model performance, as the same with the existing works  \cite{tran2021differentially},  we adopt the fairness \textit{Discrimination} ($\mathcal{G}_{ya}$) (Definition \ref{de: epsilon_f fair classifier}) and \textit{Error} (err). We test the following four questions (\textbf{Q1}-\textbf{Q3}) to analysis the FL performance under privacy and fairness constraints. 

\subsection{Performance Analysis}


Table \ref{ta: training overhead} reports the end-to-end training overhead under different system configurations. Compared to vanilla FL, enabling differential privacy alone increases the average round latency from 320 ms to 355 ms, primarily due to additional client-side perturbation and server-side noise processing. Enforcing fairness constraints introduces further overhead, resulting in an average round latency of 370 ms, as fairness evaluation and constraint handling incur extra computation on both the client and server. When privacy and fairness are jointly enforced (FedPF), the average round latency increases to 410 ms, corresponding to an overall training time of 41.0 s for 100 rounds. This represents a moderate overhead of approximately 28   \% compared to vanilla FL. Notably, the additional cost is dominated by server-side processing, reflecting the complexity of aggregating noisy updates while enforcing fairness constraints. These results demonstrate that, even on a real-world FL testbed with resource-limited clients, the proposed approach remains practical and introduces manageable system overhead while enabling simultaneous privacy protection and fairness enforcement.

\begin{figure*}[t]
    \centering
    \begin{minipage}{0.24\linewidth}
        \centering
        \centerline{
        \includegraphics[width=1\textwidth]{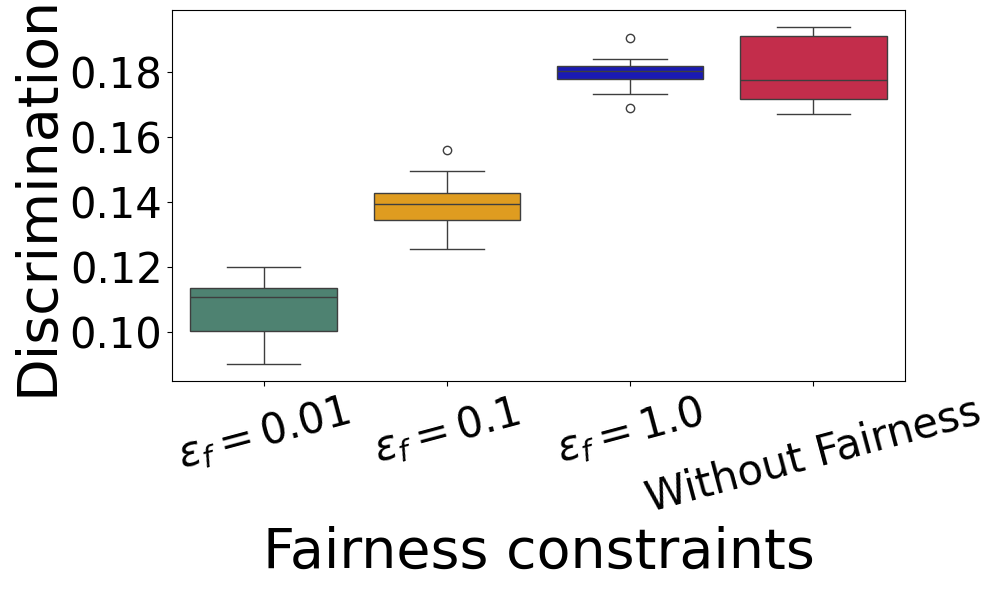}
        }
        \centerline{(\textit{a}) $\mathcal{G}_{ya}$ \textit{vs} $\varepsilon_f$ on \textit{Adult}}
    \end{minipage}
    \begin{minipage}{0.24\linewidth}
        \centering
        \centerline{
        \includegraphics[width=1\textwidth]{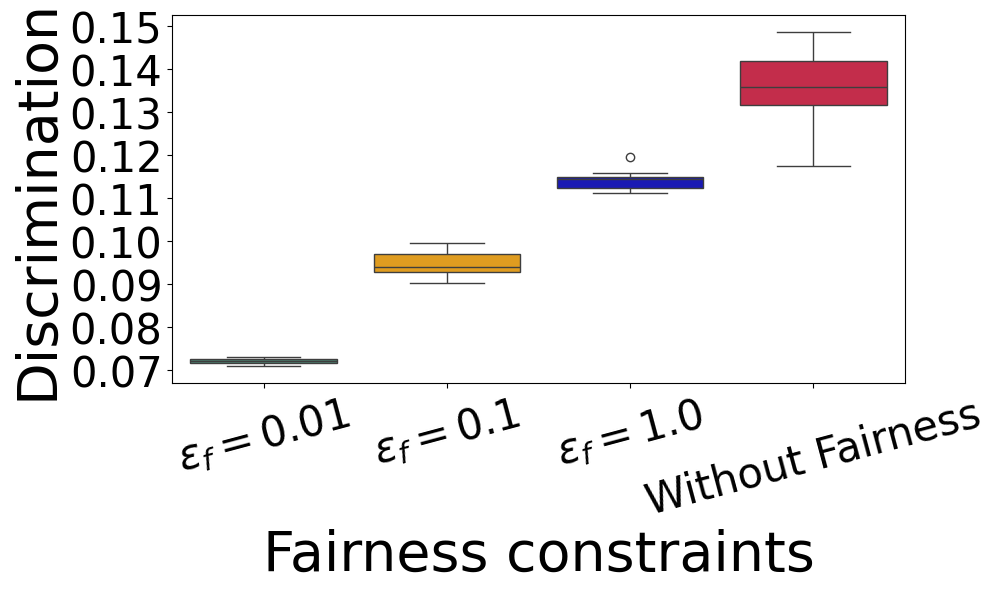}
        }
        \centerline{(\textit{b}) $\mathcal{G}_{ya}$ \textit{vs} $\varepsilon_f$ on \textit{Bank}}
    \end{minipage}
    \begin{minipage}{0.24\linewidth}
        \centering
        \centerline{
        \includegraphics[width=1\textwidth]{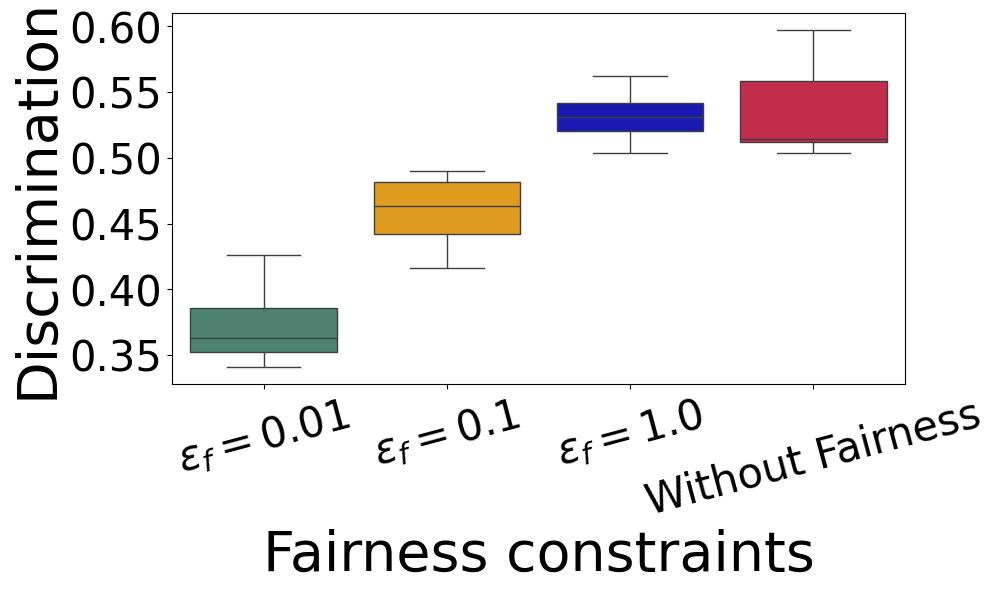}
        }
        \centerline{(\textit{c}) $\mathcal{G}_{ya}$ \textit{vs} $\varepsilon_f$ on \textit{Compas}}
    \end{minipage}
    \caption{The fairness constraints of \textit{FedPF} algorithm influence on the discrimination ($\mathcal{G}_{ya}$) without privacy protection in FL.}
    \label{fig: fair_without_privacy}
\end{figure*}

\begin{figure}[t]
    \centering
    \begin{minipage}{0.3\linewidth}
        \centering
        \centerline{
        \includegraphics[width=1\textwidth]{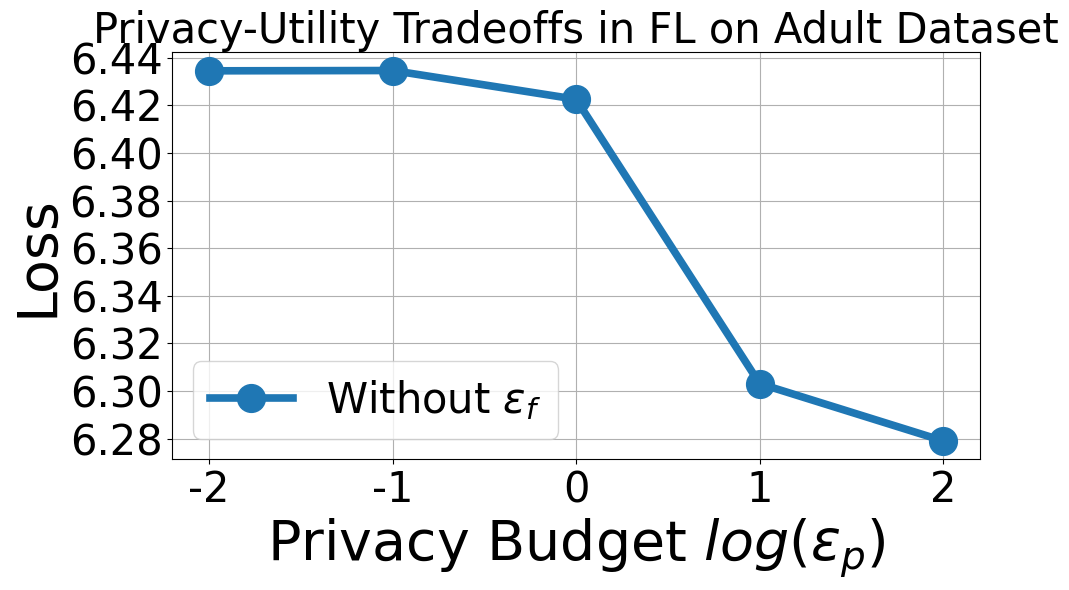}
        }
        \centerline{(\textit{a}) \textit{Loss vs} $\varepsilon_p$}
    \end{minipage}
    \begin{minipage}{0.3\linewidth}
        \centering
        \centerline{
        \includegraphics[width=1\textwidth]{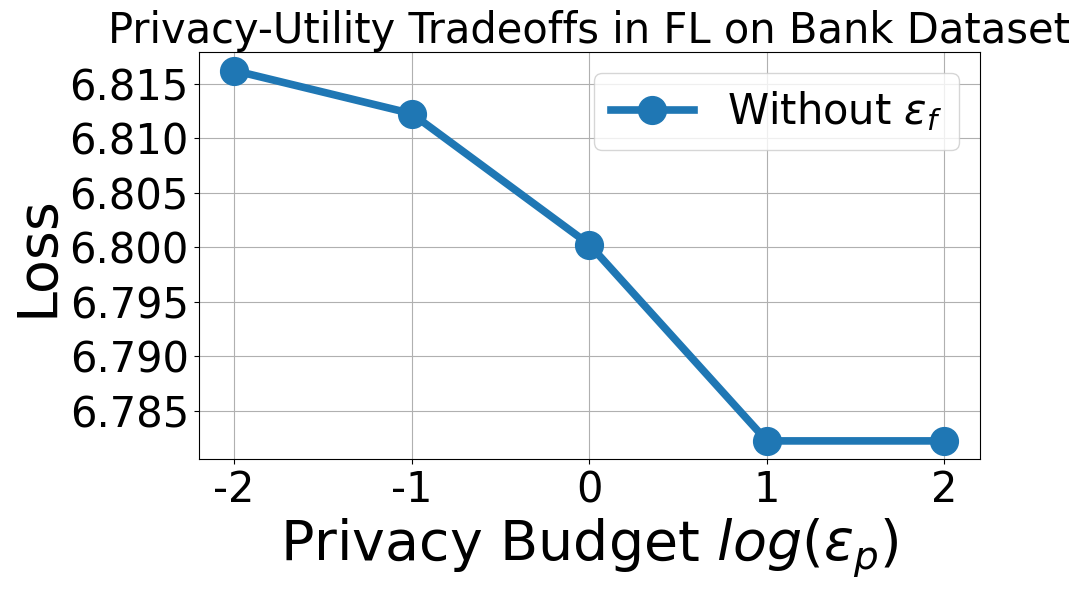}
        }
        \centerline{(\textit{b}) \textit{Loss vs} $\varepsilon_p$}
    \end{minipage}
    \begin{minipage}{0.3\linewidth}
        \centering
        \centerline{
        \includegraphics[width=1\textwidth]{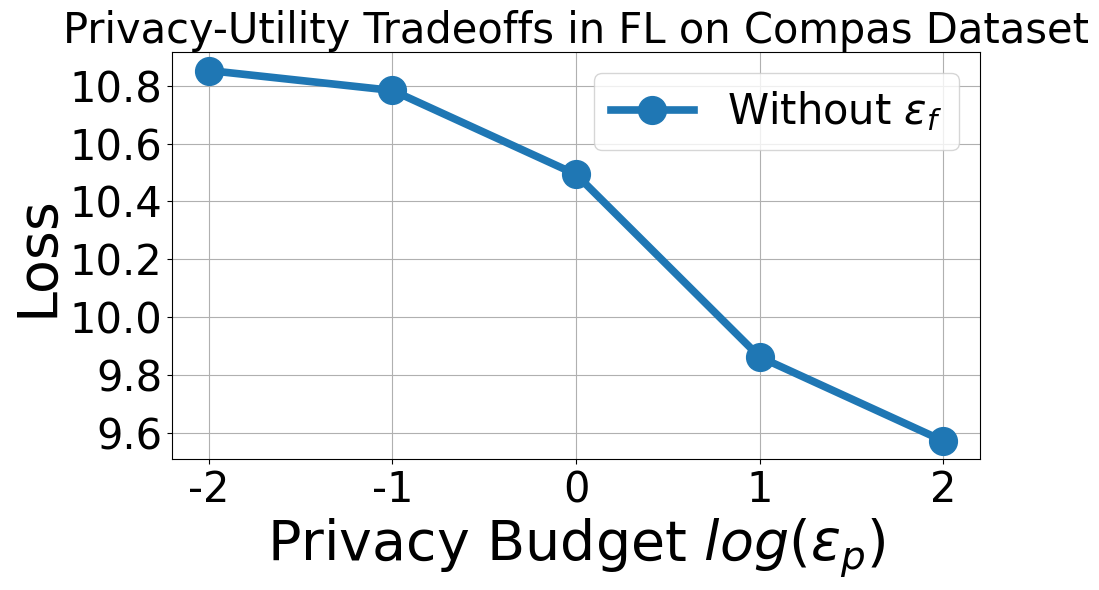}
        }
        \centerline{(\textit{c}) \textit{Loss vs} $\varepsilon_p$}
    \end{minipage}
    \caption{The privacy $\varepsilon_p$ of \textit{FedPF} algorithm influence on the loss of server model without fairness constraints in FL based on \textit{Adult}, \textit{Bank} and \textit{Compas} datasets, respectively.}
    \label{fig: loss_privacy}
\end{figure}


Table~\ref{tab: comm_overhead} reports the communication overhead per federated learning round across different methods. Vanilla FL and the DP-only variant exhibit identical communication costs, which are 512 KB uplink and downlink per client, as differential privacy is applied locally and does not alter model size. Adding fairness constraints alone increases uplink to 540 KB and downlink to 520 KB due to the larger ensemble of predictors maintained by the Exponentiated Gradient reduction in \textit{FedPF}'s fairness module. The full \textit{FedPF} method, which jointly enforces both privacy and fairness, incurs the highest overhead. It amounts to 560 KB uplink and 540 KB downlink per client, resulting in a total of 3900 KB per round, an approximately 8.8\% increase over vanilla FL. This modest overhead is primarily attributable to the weighted ensemble representation in the fairness-aware local models and remains practical for deployment on resource-constrained devices, demonstrating that \textit{FedPF} achieves substantial improvements in privacy-fairness trade-offs with only marginal additional communication cost.

In Table \ref{ta: comparison of fairness and privacy}, we compared the performance of different algorithms on various datasets, specifically including the model's accuracy and fairness metrics. We compared the proposed \textit{FedPF} algorithm with other baseline algorithms to verify its effectiveness under privacy protection and fairness constraints. Based on the average results from three independent experiments, \textit{FedPF} (Ours) demonstrates a superior ability to resolve the complex trilemma between privacy protection, fairness, and model utility. The data clearly indicates that \textit{FedPF} outperforms state-of-the-art baselines (such as FedAvg, FedAA, and CENTAUR) in fairness optimization by a significant margin. Specifically, under a high privacy budget ($log(\varepsilon) = 2$), while traditional methods suffer from extreme demographic bias with an discrimination ($\mathcal{G}_{ya}$) near 1.0, \textit{FedPF} successfully reduces this disparity to 0.3688, a reduction in bias of approximately 60\%. Remarkably, even under strict privacy constraints, \textit{FedPF} maintains the lowest bias levels across all tested algorithms while preserving a high utility of approximately 84\% accuracy. These findings validate that our game-theoretic dual-update mechanism (\textit{FedPF}) effectively captures and corrects demographic biases, establishing \textit{FedPF} as a new benchmark for achieving \textit{private and fair} collaborative training without compromising core model performance in federated learning framework.

\begin{figure}[t]
     \centering
    \begin{minipage}{0.3\linewidth}
        \centering
        \centerline{
        \includegraphics[width=1\textwidth]{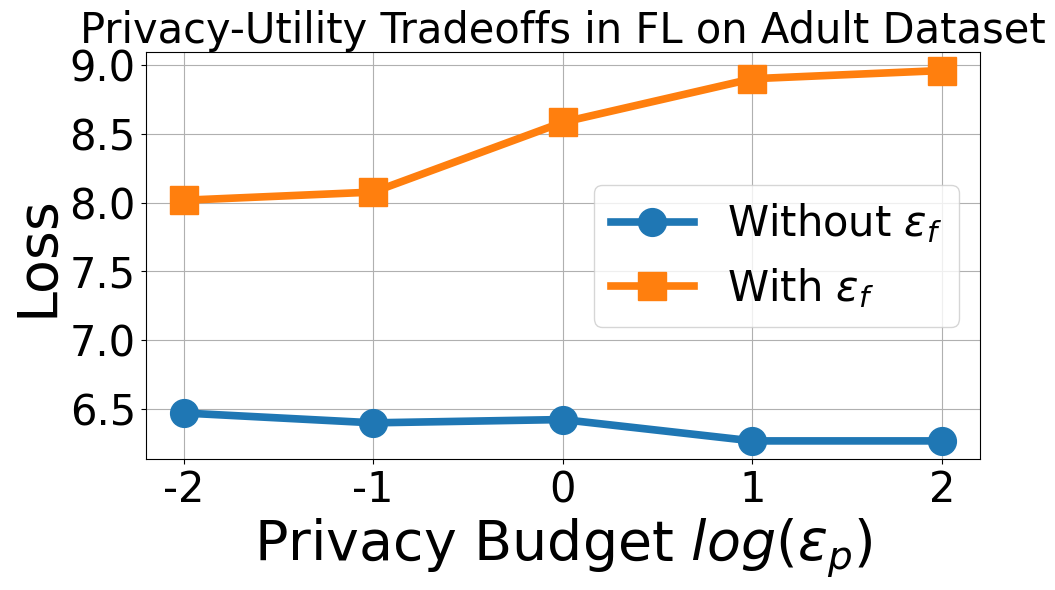}
        }
        \centerline{(\textit{a}) \textit{Loss} \textit{vs} $\varepsilon_p$}
    \end{minipage}
    \begin{minipage}{0.3\linewidth}
        \centering
        \centerline{
        \includegraphics[width=1\textwidth]{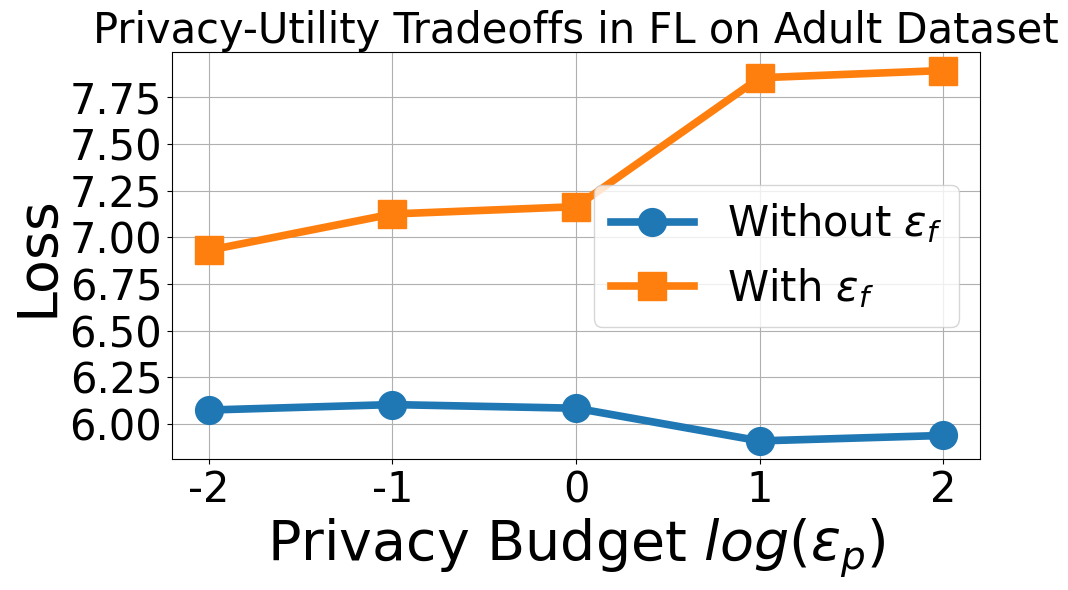}
        }
        \centerline{(\textit{b}) \textit{Loss} \textit{vs} $\varepsilon_p$}
    \end{minipage}
    \begin{minipage}{0.3\linewidth}
        \centering
        \centerline{
        \includegraphics[width=1\textwidth]{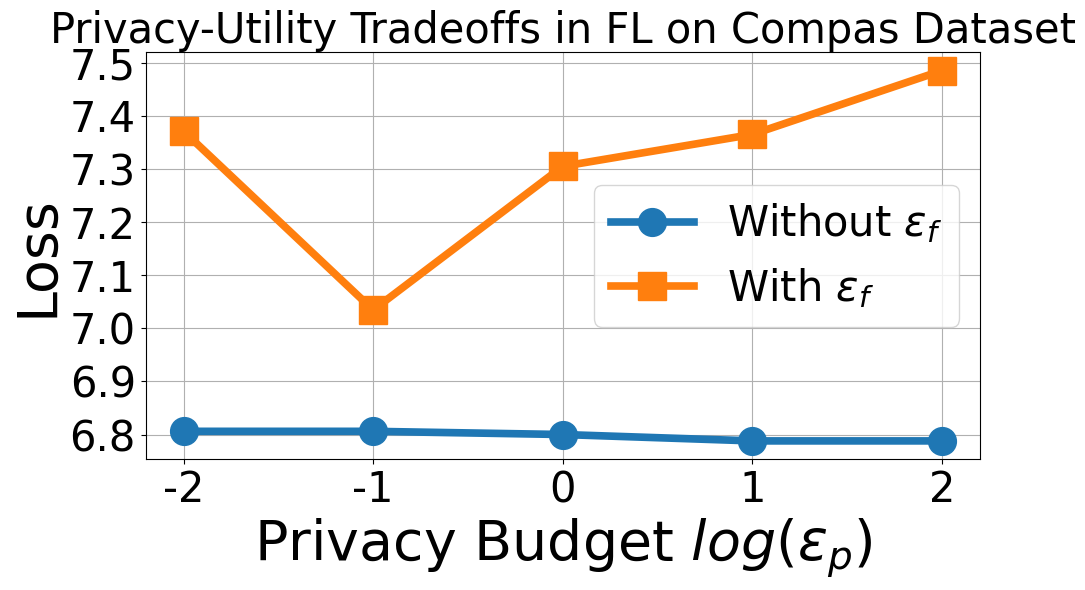}
        }
        \centerline{(\textit{c}) \textit{Loss} \textit{vs} $\varepsilon_p$}
    \end{minipage}

    \begin{minipage}{0.3\linewidth}
        \centering
        \centerline{
        \includegraphics[width=1\textwidth]{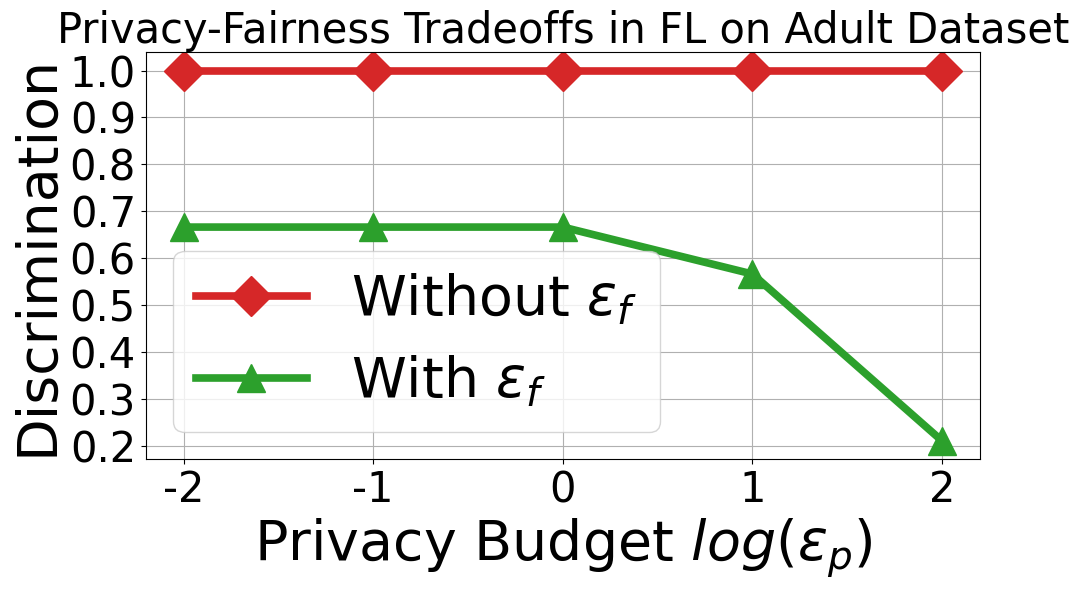}
        }
        \centerline{(\textit{d}) $\mathcal{G}_{ya}$ \textit{vs} $\varepsilon_p$}
    \end{minipage}
    \begin{minipage}{0.3\linewidth}
        \centering
        \centerline{
        \includegraphics[width=1\textwidth]{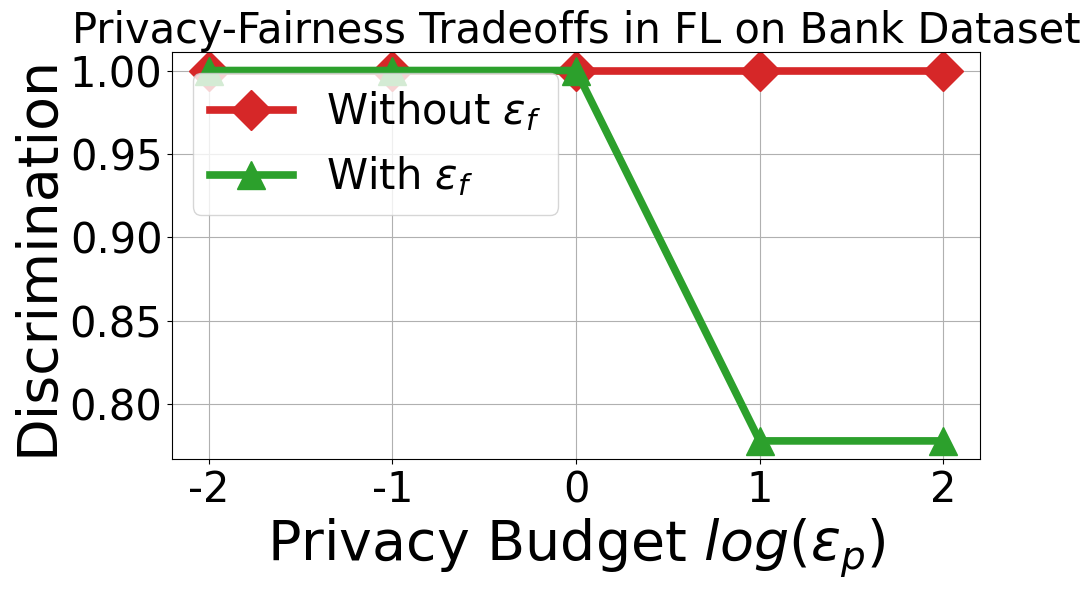}
        }
        \centerline{(\textit{e}) $\mathcal{G}_{ya}$ \textit{vs} $\varepsilon_p$}
    \end{minipage}
    \begin{minipage}{0.3\linewidth}
        \centering
        \centerline{
        \includegraphics[width=1\textwidth]{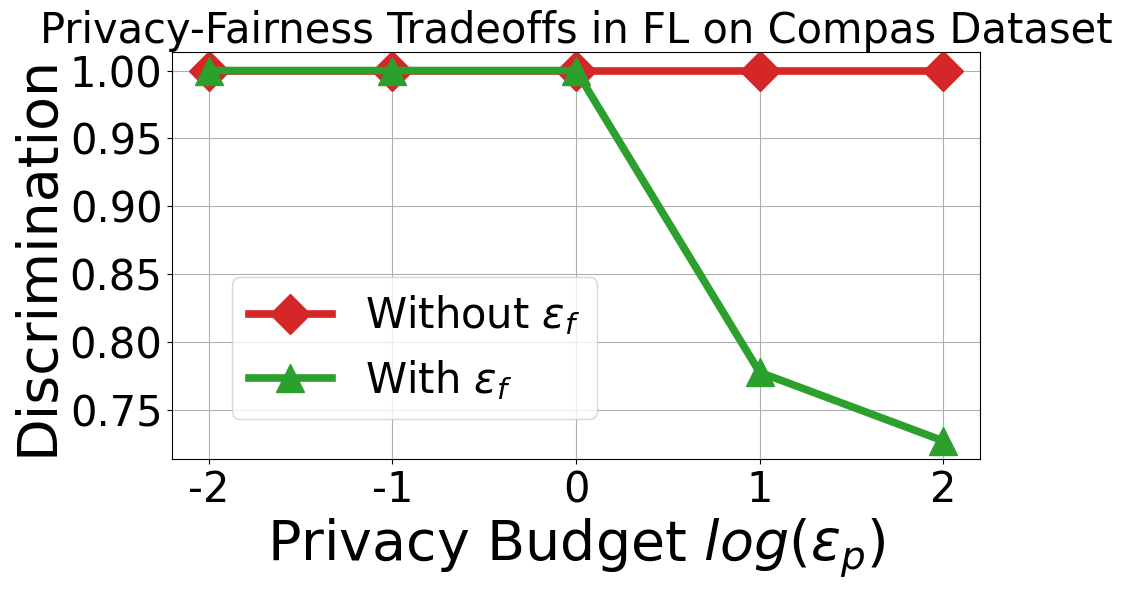}
        }
        \centerline{(\textit{f}) $\mathcal{G}_{ya}$ \textit{vs} $\varepsilon_p$}
    \end{minipage}
    \caption{The privacy budget of \textit{FedPF} algorithm influence on the loss and the discrimination (\textit{EO}) of server model in FL based on \textit{FedPF} algorithm. The fairness constraints include \textit{without fairness constraints} and \textit{with fairness constraints} ($\varepsilon_f = 0.1$) lines. The sensitive attributes in \textit{Adult}, \textit{Bank} and \textit{Compas} datasets are \textit{Age}, \textit{Age} and \textit{Sex}, respectively.}
    \label{fig: fair_privacy_utility}
\end{figure}

\subsection{Impact of Fairness and Privacy Constraints of \textit{FedPF}}

\textbf{Q1}: \textit{How do fairness constraints affect discrimination without privacy considerations?}

To isolate the effect of fairness constraints, we configure \textit{FedPF} without privacy protection ($\varepsilon_p \rightarrow \infty$) and vary the fairness tolerance parameter $\varepsilon_f \in \{0.01, 0.1, 1.0\}$, comparing against a baseline \textit{without fairness} constraints. Fig. \ref{fig: fair_without_privacy} demonstrates the discrimination reduction achieved by fairness constraints across all three datasets. The results show a consistent inverse relationship between fairness constraint strength and discrimination levels. When $\varepsilon_f =0.01$, the \textit{FedPF} algorithmic discrimination ($\mathcal{G}_{ya}$) of the server model decreased by 38.9\%, 42.9\%, and 29.4\% on the \textit{Adult}, \textit{Bank}, and \textit{Compas} datasets, respectively, compared to the case with no fairness constraint.

\textbf{Key Observation 1}: Fig. \ref{fig: fair_without_privacy} shows that \textit{stricter fairness constraints (smaller $\varepsilon_f$) consistently lead to more significant discrimination reduction across all datasets.} These results provide empirical validation of fairness-utility tradeoff predictions in Theorem \ref{th: tradeoffs}. The consistent discrimination reduction demonstrates that our fairness constraints effectively suppress discriminatory patterns.


\textbf{Q2}: \textit{How does privacy budget affect model utility without fairness constraints?}

We examine the privacy-utility relationship by fixing fairness constraints to be inactive ($\varepsilon_f \rightarrow \infty$) and varying the privacy budget $\varepsilon_p \in [0.01, 100]$ across our three datasets. Fig. \ref{fig: loss_privacy} illustrates the relationship between privacy budget and model loss. Consistent with differential privacy theory, we observe an inverse correlation between privacy cost and model performance:

\textbf{Key Observation 2}: Fig. \ref{fig: loss_privacy} shows that as the privacy cost increases, without considering fairness, the loss of the FL server model gradually decreases. The privacy cost is inversely proportional to the loss of the model, which is consistent with the theoretical derivation in Theorem \ref{th: tradeoffs}.

\subsection{Privacy-Fairness-Utility Tradeoff Analysis of \textit{FedPF}}
\textbf{Q3}: \textit{How do privacy and fairness constraints interact to affect overall system performance?}

Now we examine examines the complex three-way interaction between privacy, fairness, and utility by jointly varying both $\varepsilon_p$ and $\varepsilon_f$ parameters, revealing non-intuitive relationships that extend beyond simple pairwise tradeoffs. Fig. \ref{fig: fair_privacy_utility} reveals a surprising non-monotonic relationship between privacy budget and model loss when fairness constraints are active. We observe two distinct behavioral regimes. 
\begin{itemize}
    \item \textbf{Low $\varepsilon_p$}: When privacy budget is limited, model loss decreases in $\varepsilon_p$. In this regime, the privacy-fairness coupling term in Theorem \ref{th: tradeoffs} and Theorem \ref{th: non-iid tradeoffs} dominates the total error.
    \item \textbf{High $\varepsilon_p$}: When privacy budget exceeds a threshold, model loss begins to increase in $\varepsilon_p$. The generalization error ($e_G$) term becomes dominant in Theorem \ref{th: tradeoffs} and Theorem \ref{th: non-iid tradeoffs}.
\end{itemize}

\textbf{Key Observation 3}: \textit{The fairness constraint fundamentally alters the privacy-utility relationship, creating a non-monotonic curve where optimal utility occurs at intermediate privacy levels rather than maximum privacy budgets.}

\section{Conclusions}\label{sec: conclusion}

This paper addresses the critical challenge of simultaneously achieving privacy, fairness, and utility in FL. We introduce the $\varepsilon_f$-fair constraint for FL environments and develop the \textit{FedPF} algorithm through Lagrangian min-max optimization. Our theoretical analysis establishes fundamental privacy-fairness-utility tradeoffs (Theorem \ref{th: tradeoffs} and Theorem \ref{th: non-iid tradeoffs}) and convergence guarantees (Theorem \ref{th: fairness_constraints}). Most significantly, we reveal counter-intuitive insights, i.e., fairness constraints create non-monotonic privacy-utility relationships where optimal performance occurs at intermediate privacy budgets. Experimental validation on \textit{Adult}, \textit{Bank}, and \textit{Compas} datasets confirms these predictions, demonstrating up to 42.9\% discrimination reduction while maintaining differential privacy guarantees. And, hardware-level simulations demonstrate that FedPF maintains a low computational footprint, making it suitable for resource-constrained edge devices.

\textit{Limitations.} Our hardware testbed includes 6 clients, which validates feasibility but limits conclusions about large-scale heterogeneous deployments. Theorem \ref{th: non-iid tradeoffs}'s $\Gamma T/N$ term suggests that increasing client count $N$ can mitigate non-i.i.d. drift; we plan to validate this on 50+ clients in healthcare and financial federations as future work.



\small

\bibliography{reference}

@inproceedings{mcmahan2017communication,
  title={Communication-efficient learning of deep networks from decentralized data},
  author={McMahan, Brendan and Moore, Eider and Ramage, Daniel and Hampson, Seth and y Arcas, Blaise Aguera},
  booktitle={Artificial Intelligence and Statistics},
  pages={1273--1282},
  year={2017},
  organization={PMLR}
}

@article{dwork2014algorithmic,
  title={The algorithmic foundations of differential privacy},
  author={Dwork, Cynthia and Roth, Aaron and others},
  journal={Foundations and Trends{\textregistered} in Theoretical Computer Science},
  volume={9},
  number={3--4},
  pages={211--407},
  year={2014},
  publisher={Now Publishers, Inc.}
}

@article{hardt2016equality,
  title={Equality of opportunity in supervised learning},
  author={Hardt, Moritz and Price, Eric and Srebro, Nati},
  journal={Advances in Neural Information Processing Systems},
  volume={29},
  year={2016}
}

@inproceedings{laakom2025fairness,
  title={Fairness Overfitting in Machine Learning: An Information-Theoretic Perspective},
  author={Laakom, Firas and Chen, Haobo and Schmidhuber, J{\"u}rgen and Bu, Yuheng},
  booktitle={International Conference on Machine Learning},
  pages={32078--32115},
  year={2025},
  organization={PMLR}
}

@inproceedings{corbucci2024puffle,
  title={PUFFLE: balancing privacy, utility, and fairness in federated learning},
  author={Corbucci, Luca and Heikkil{\"a}, Mikko A and Noguero, David Solans and Monreale, Anna and Kourtellis, Nicolas},
  booktitle={ECAI 2024: 27th European Conference on Artificial Intelligence, 19--24 October 2024, Santiago de Compostela, Spain--Including 13th Conference on Prestigious Applications of Intelligent Systems (PAIS 2024)},
  pages={1639--1647},
  year={2024},
  organization={SAGE Publications Pvt. Ltd 1 Oliver's Yard, 55 City Road, London, EC1Y 1SP}
}

@article{gu2022privacy,
  title={Privacy, accuracy, and model fairness trade-offs in federated learning},
  author={Gu, Xiuting and Tianqing, Zhu and Li, Jie and Zhang, Tao and Ren, Wei and Choo, Kim-Kwang Raymond},
  journal={Computers \& Security},
  volume={122},
  pages={102907},
  year={2022},
  publisher={Elsevier}
}

@inproceedings{sun2023toward,
  title={Toward the tradeoffs between privacy, fairness and utility in federated learning},
  author={Sun, Kangkang and Zhang, Xiaojin and Lin, Xi and Li, Gaolei and Wang, Jing and Li, Jianhua},
  booktitle={International Symposium on Emerging Information Security and Applications},
  pages={118--132},
  year={2023},
  organization={Springer}
}

@inproceedings{li2019fair,
  title={Fair Resource Allocation in Federated Learning},
  author={Li, Tian and Sanjabi, Maziar and Beirami, Ahmad and Smith, Virginia},
  booktitle={International Conference on Learning Representations}
}

@inproceedings{martinez2020minimax,
  title={Minimax pareto fairness: A multi objective perspective},
  author={Martinez, Natalia and Bertran, Martin and Sapiro, Guillermo},
  booktitle={International Conference on Machine Learning},
  pages={6755--6764},
  year={2020},
  organization={PMLR}
}

@inproceedings{yu2020fairness,
  title={A fairness-aware incentive scheme for federated learning},
  author={Yu, Han and Liu, Zelei and Liu, Yang and Chen, Tianjian and Cong, Mingshu and Weng, Xi and Niyato, Dusit and Yang, Qiang},
  booktitle={Proceedings of the AAAI/ACM Conference on AI, Ethics, and Society},
  pages={393--399},
  year={2020}
}

@article{kusner2017counterfactual,
  title={Counterfactual fairness},
  author={Kusner, Matt J and Loftus, Joshua and Russell, Chris and Silva, Ricardo},
  journal={Advances in Neural Information Processing Systems},
  volume={30},
  year={2017}
}

@article{kairouz2021advances,
  title={Advances and open problems in federated learning},
  author={Kairouz, Peter and McMahan, H Brendan and Avent, Brendan and Bellet, Aur{\'e}lien and Bennis, Mehdi and Bhagoji, Arjun Nitin and Bonawitz, Kallista and Charles, Zachary and Cormode, Graham and Cummings, Rachel and others},
  journal={Foundations and Trends{\textregistered} in Machine Learning},
  volume={14},
  number={1--2},
  pages={1--210},
  year={2021},
  publisher={Now Publishers, Inc.}
}

@article{shao2023survey,
  title={A survey of what to share in federated learning: Perspectives on model utility, privacy leakage, and communication efficiency},
  author={Shao, Jiawei and Li, Zijian and Sun, Wenqiang and Zhou, Tailin and Sun, Yuchang and Liu, Lumin and Lin, Zehong and Mao, Yuyi and Zhang, Jun},
  journal={arXiv preprint arXiv:2307.10655},
  year={2023}
}

@inproceedings{hao2021towards,
  title={Towards fair federated learning with zero-shot data augmentation},
  author={Hao, Weituo and El-Khamy, Mostafa and Lee, Jungwon and Zhang, Jianyi and Liang, Kevin J and Chen, Changyou and Duke, Lawrence Carin},
  booktitle={Proceedings of the IEEE/CVF Conference on Computer Vision and Pattern Recognition},
  pages={3310--3319},
  year={2021}
}

@article{duan2020self,
  title={Self-balancing federated learning with global imbalanced data in mobile systems},
  author={Duan, Moming and Liu, Duo and Chen, Xianzhang and Liu, Renping and Tan, Yujuan and Liang, Liang},
  journal={IEEE Transactions on Parallel and Distributed Systems},
  volume={32},
  number={1},
  pages={59--71},
  year={2020},
  publisher={IEEE}
}

@article{xu2021privacy,
  title={Privacy-preserving machine learning: Methods, challenges and directions},
  author={Xu, Runhua and Baracaldo, Nathalie and Joshi, James},
  journal={arXiv preprint arXiv:2108.04417},
  year={2021}
}

@article{li2025clients,
  title={Clients collaborate: Flexible differentially private federated learning with guaranteed improvement of utility-privacy trade-off},
  author={Li, Yuecheng and Fu, Lele and Wang, Tong and Lou, Jian and Chen, Bin and Yang, Lei and Shen, Jian and Zheng, Zibin and Chen, Chuan},
  journal={Proceedings of the 42nd International Conference on Machine Learning},
  year={2025}
}

@inproceedings{he2025fedaa,
  title={FedAA: A Reinforcement Learning Perspective on Adaptive Aggregation for Fair and Robust Federated Learning},
  author={He, Jialuo and Chen, Wei and Zhang, Xiaojin},
  booktitle={Proceedings of the AAAI Conference on Artificial Intelligence},
  volume={39},
  number={16},
  pages={17085--17093},
  year={2025}
}

@article{shen2023share,
  title={Share your representation only: Guaranteed improvement of the privacy-utility tradeoff in federated learning},
  author={Shen, Zebang and Ye, Jiayuan and Kang, Anmin and Hassani, Hamed and Shokri, Reza},
  journal={arXiv preprint arXiv:2309.05505},
  year={2023}
}

@inproceedings{farrand2020neither,
  title={Neither private nor fair: Impact of data imbalance on utility and fairness in differential privacy},
  author={Farrand, Tom and Mireshghallah, Fatemehsadat and Singh, Sahib and Trask, Andrew},
  booktitle={Proceedings of the 2020 Workshop on Privacy-preserving Machine Learning in Practice},
  pages={15--19},
  year={2020}
}

@inproceedings{ganev2022robin,
  title={Robin hood and matthew effects: Differential privacy has disparate impact on synthetic data},
  author={Ganev, Georgi and Oprisanu, Bristena and De Cristofaro, Emiliano},
  booktitle={International Conference on Machine Learning},
  pages={6944--6959},
  year={2022},
  organization={PMLR}
}

@article{bagdasaryan2019differential,
  title={Differential privacy has disparate impact on model accuracy},
  author={Bagdasaryan, Eugene and Poursaeed, Omid and Shmatikov, Vitaly},
  journal={Advances in Neural Information Processing Systems},
  volume={32},
  year={2019}
}

@inproceedings{esipova2022disparate,
  title={Disparate Impact in Differential Privacy from Gradient Misalignment},
  author={Esipova, Maria S and Ghomi, Atiyeh Ashari and Luo, Yaqiao and Cresswell, Jesse C},
  booktitle={The Eleventh International Conference on Learning Representations}
}

@article{berk2021fairness,
  title={Fairness in criminal justice risk assessments: The state of the art},
  author={Berk, Richard and Heidari, Hoda and Jabbari, Shahin and Kearns, Michael and Roth, Aaron},
  journal={Sociological Methods \& Research},
  volume={50},
  number={1},
  pages={3--44},
  year={2021},
  publisher={Sage Publications Sage CA: Los Angeles, CA}
}

@inproceedings{agarwal2018reductions,
  title={A reductions approach to fair classification},
  author={Agarwal, Alekh and Beygelzimer, Alina and Dud{\'\i}k, Miroslav and Langford, John and Wallach, Hanna},
  booktitle={International Conference on Machine Learning},
  pages={60--69},
  year={2018},
  organization={PMLR}
}

@inproceedings{jagielski2019differentially,
  title={Differentially private fair learning},
  author={Jagielski, Matthew and Kearns, Michael and Mao, Jieming and Oprea, Alina and Roth, Aaron and Sharifi-Malvajerdi, Saeed and Ullman, Jonathan},
  booktitle={International Conference on Machine Learning},
  pages={3000--3008},
  year={2019},
  organization={PMLR}
}

@inproceedings{mozannar2020fair,
  title={Fair learning with private demographic data},
  author={Mozannar, Hussein and Ohannessian, Mesrob and Srebro, Nathan},
  booktitle={International Conference on Machine Learning},
  pages={7066--7075},
  year={2020},
  organization={PMLR}
}

@inproceedings{tran2021differentially,
  title={Differentially private and fair deep learning: A lagrangian dual approach},
  author={Tran, Cuong and Fioretto, Ferdinando and Van Hentenryck, Pascal},
  booktitle={Proceedings of the AAAI Conference on Artificial Intelligence},
  volume={35},
  number={11},
  pages={9932--9939},
  year={2021}
}

@online{homeb,
  title = {Home},
  url = {https://ieee-p21451-1-5.github.io/},
  organization = {IEEE P21451-1-5 Working Group}
}

@article{wang2020robust,
  title={Robust optimization for fairness with noisy protected groups},
  author={Wang, Serena and Guo, Wenshuo and Narasimhan, Harikrishna and Cotter, Andrew and Gupta, Maya and Jordan, Michael},
  journal={Advances in Neural Information Processing Systems},
  volume={33},
  pages={5190--5203},
  year={2020}
}

@article{edwards2011kantorovich,
  title={On the kantorovich--rubinstein theorem},
  author={Edwards, David A},
  journal={Expositiones Mathematicae},
  volume={29},
  number={4},
  pages={387--398},
  year={2011},
  publisher={Elsevier}
}

@inproceedings{zhang2025fedlth,
  title={FedLTH: A Privacy-preserving Federated Learning Framework with Model Pruning on Edge Clients},
  author={Zhang, Heyu and Xie, Yulai and Hu, Shengshan and He, Miao and He, Peisong and Zheng, Jun and Feng, Dan},
  booktitle={2025 IEEE 45th International Conference on Distributed Computing Systems (ICDCS)},
  pages={384--394},
  year={2025},
  organization={IEEE}
}

@inproceedings{he2025pp,
  title={PP-FCL: Privacy-Preserving Federated Continual Learning via Generative Replay and Incremental Representation Enhancement},
  author={He, Zaobo and Wang, Yunkun and Cai, Zhipeng and Li, Yingshu},
  booktitle={2025 IEEE 45th International Conference on Distributed Computing Systems (ICDCS)},
  pages={780--790},
  year={2025},
  organization={IEEE}
}
\bibliographystyle{IEEEtran}

\clearpage

\section{Appendix}

\begin{theorem1}[Privacy-Fairness-Utility Tradeoff of \textit{FedPF}]
    Let \(\hat{Y}\) be the output of the \textit{FedPF} algorithm, and \(Y^*\) be the optimal classifier satisfying \((\varepsilon_p, \delta)\)-differential privacy (DP) and \(\varepsilon_f\)-fairness constraints (EO, DemP). If client data are independently and identically distributed and the Rademacher complexity of the classifier family is bounded, with probability at least \(1-\beta\), the following inequality holds:
    \begin{equation}
        \footnotesize
        \begin{aligned}
            \mathrm{err}(\hat{Y}) \leq \mathrm{err}(Y^*) &+ O \left( \frac{B^2 \varepsilon_f^4 T^{3/2} \mathcal{H}}{\varepsilon_p^2} \right) + O\left( \mathfrak{R}_m(\mathcal{F}) + \sqrt{\frac{\log(1/\beta)}{m}}\right),\\
            \mathrm{err}(\hat{Y}) \geq \mathrm{err}(Y^*) &+ \Omega\left( \frac{|\mathcal{A}| \log(1/\varepsilon_f)}{\sqrt{m}} - \mathfrak{R}_m(\mathcal{F}) \right) + \Omega\left( \sqrt{\frac{\log(1/\beta)}{m}} \right),
        \end{aligned}
    \end{equation}
    where \(\mathcal{H} = \ln(1/\delta) \ln^2(8T|\mathcal{A}|/\delta) \log(|\mathcal{K}| + 1)\). $\mathfrak{R}_m(\mathcal{F})$ is Rademacher complexity of the
classifier family $\mathcal{F}$, \(\mathfrak{R}_m(\mathcal{F}) \leq C m^{-\alpha}\), \(\alpha \leq 1/2\). \(B\) is the \(\ell_1\)-norm bound of dual variables \(\lambda\), \(T\) is the number of training rounds. \(m\) is the client-side sample size. \(|\mathcal{A}|\) is the number of sensitive attribute categories. $\mathcal{K}$ is a set of indices, and each index $k \in \mathcal{K}$ corresponds to a fairness constraint. If fairness constraints is \textit{DemP}, then $|\mathcal{K}| = 2|\mathcal{A}|$, if fairness constraints is \textit{EO}, then $|\mathcal{K}| = 4|\mathcal{A}|$. $\delta$ is the parameter of DP mechanism.
    \label{th: tradeoffs}
\end{theorem1}

\begin{proof}[Sketch]
The \textit{FedPF} algorithm formulates the optimization as a zero-sum game between a learner (minimizing loss) and an auditor (enforcing fairness) under DP constraints. We derive the error bound as follows:

For the \textbf{upper bound}, the suboptimality gap \( \operatorname{err}(\hat{Y}) - \operatorname{err}(Y^*) \) can be decomposed into Learner Regret and Auditor Regret from work \cite{jagielski2019differentially}. However, The influence of the "fairness overfitting" phenomenon for model error cannot be ignored \cite{laakom2025fairness}. Thus, the error bound is considered Learner Regret ({\color{blue}$e_{LR}$}) and Auditor Regret ({\color{blue}$e_{AR}$}) and Generalization Error ({\color{blue}$e_G$}), as follows:
\begin{equation}
    \footnotesize
    \begin{aligned}
        \operatorname{err}(Y) - \operatorname{err}(Y^*) &\leq \left( \frac{1}{T} \sum_{t=1}^{T} \mathcal{L}(f, \lambda) - \frac{1}{T} \min_{f \in \mathcal{F}} \sum_{t=1}^{T}\mathcal{L}(f, \lambda) \right) ({\color{blue}e_{LR}}) \\[0.1cm]
        &+ \left(\frac{1}{T} \max_{\lambda \in \Lambda} \sum_{t=1}^{T} \mathcal{L}(f, \lambda) - \frac{1}{T} \sum_{t=1}^{T}\mathcal{L}(f, \lambda) \right) (\color{blue}e_{AR}) \\[0.1cm]
        &+ \left[ \operatorname{err}(f) - \operatorname{err}(f^*) \right] (\color{blue}e_G)
    \end{aligned}
\end{equation}

1). \textit{Learner Regret}: The learner minimizes a Lagrangian \( \mathcal{L}(f, \lambda) = \mathbb{E}[\ell(f(x), y)] + \lambda \cdot \operatorname{disc}(f) \), where \( \lambda \) are dual variables for fairness constraints. Using DP (Exponential Mechanism definition \ref{de: exponential mechanism}), noise scaled by \( \varepsilon_p \) is added to sensitive dataset. We can state the learner regret with probability at least \(1 - \beta/2\) (in Lemma 4.1 \cite{jagielski2019differentially}):
\begin{equation}
    \small
    \begin{aligned}
        {\color{blue}e_{LR}} = O\left( \frac{B \varepsilon_f^2 \sqrt{T \ln(1/\delta)} \ln m}{\varepsilon_p} \right).
    \end{aligned}
    \label{eq: LR}
\end{equation}

2. \textit{Auditor Regret}: The auditor updates \( \lambda \) to enforce \( \operatorname{disc}(f) \leq \varepsilon_f \). We can state the auditor regret with probability at least \(1 - \beta/2\) (in Lemma 4.2 \cite{jagielski2019differentially}):
\begin{equation}
    \small
    \begin{aligned}
        {\color{blue}e_{AR}} = O\left( \frac{B^2 \varepsilon_f^4 T^{3/2} \ln(1/\delta) \ln^2(8T|A|/\delta) \log(|\mathcal{K}| + 1)}{\varepsilon_p^2} \right).
    \end{aligned}
    \label{eq: AR}
\end{equation}

3. \textit{Generalization Error}: Using Rademacher complexity and standard concentration bounds \cite{agarwal2018reductions}, the generalization error under fairness constraint (EO and DemP) is:
\begin{equation}
    \small
    {\color{blue}e_G} = O\left( \mathfrak{R}_m(\mathcal{F}) + \sqrt{\frac{\log(1/\beta)}{m}} \right).
    \label{eq: GE}
\end{equation}

Summing these terms yields the error bound. It can be seen from (\ref{eq: LR}) and (\ref{eq: AR}) that the impact of Auditor Regret is greater than that of Learner Regret. 

For the \textbf{lower bound},  Fairness requires estimating group statistics $\gamma_{y,a}(f)$ (Definition 3) with precision $\varepsilon_f$. By Hoeffding's inequality \cite{agarwal2018reductions}, estimating differences $|\gamma_{y,a} - \gamma_{y,a'}|$ over $|\mathcal{A}|$ groups requires $\Omega(|\mathcal{A}|/\varepsilon_f^2)$ samples per group for variance control. With total $m$ samples, the minimal detection error induces loss $\Omega(|\mathcal{A}|/(m\varepsilon_f^2))$. Taking log for entropy, we get $\Omega(|\mathcal{A}| \log(1/\varepsilon_f)/\sqrt{m})$ as the minimal utility penalty for strict $\varepsilon_f$. Moderate $\varepsilon_f$ acts as regularization, reducing overfitting \cite{laakom2025fairness}. This subtracts from loss via Rademacher bound: $\text{gain} \leq -\mathfrak{R}_m(\mathcal{F})$ (negative, as fairness constraints prune complex discriminators). For strict $\varepsilon_f$, the penalty dominates, making the bound positive.  Apply McDiarmid's inequality for concentration: the bound holds w.p. $1 - \beta$ with additive $\Omega\left( \sqrt{\log(1/\beta)/m} \right)$ \cite{agarwal2018reductions,mozannar2020fair}. Sum terms, yielding the lower bound. This matches the paper's non-monotonicity (bound dips then rises as $\varepsilon_f \downarrow$) and implies impossibility: no algorithm achieves $\Delta \text{err} < \Omega(\log(1/\varepsilon_f)/\sqrt{m})$ for small $\varepsilon_f$.

\end{proof}

\begin{theorem1}[Extended for Non-I.I.D. of Theorem 1]
\label{th: non-iid tradeoffs}
Let $\hat{Y}$ be the output of the \textit{FedPF} algorithm, and $Y^*$ be the optimal classifier satisfying $(\varepsilon_p, \delta)$-differential privacy (DP) and $\varepsilon_f$-fairness constraints (EO, DemP). Assume non-i.i.d. data with bounded heterogeneity $\Gamma = \max_{i,j} \|\nabla \mathcal{L}(D_i) - \nabla \mathcal{L}(D_j)\|^2$ (gradient dissimilarity across clients) and bounded local variance $\sigma^2 = \mathbb{E}_i[\text{Var}(\nabla \mathcal{L}(D_i))]$. If the Rademacher complexity of the classifier family is bounded, with probability at least $1 - \beta$, the following inequalities hold:
\begin{equation}
\footnotesize
\begin{aligned}
&\operatorname{err}(\hat{Y}) \leq \operatorname{err}(Y^*) + O\left(\mathcal{A}\right) + O\left(\mathcal{B} + \frac{\Gamma T}{N} + \sigma^2 \sqrt{\frac{T}{mN}}\right), \\
&\operatorname{err}(\hat{Y}) \geq \operatorname{err}(Y^*) + \Omega\left(\mathcal{C} + \frac{\Gamma T}{N} + \sigma^2 \sqrt{\frac{T}{mN}}\right) + \Omega\left(\sqrt{\frac{\log(1/\beta)}{m}} \right),
\end{aligned}
\end{equation}
where all terms match Theorem 1, $\mathcal{A} = \frac{B^2 \varepsilon_f^4 T^{3/2} \mathcal{H}}{\varepsilon_p^2}$, $\mathcal{B} =  \mathfrak{R}_m(\mathcal{F}) + \sqrt{\frac{\log(1/\beta)}{m}}$, $\mathcal{C}=\frac{|\mathcal{A}| \log(1/\varepsilon_f)}{\sqrt{m}} - \mathfrak{R}_m(\mathcal{F})$, plus non-i.i.d. penalties: $\frac{\Gamma T}{N}$ (heterogeneity drift over $T$ rounds, averaged over $N$ clients) and $\sigma^2 \sqrt{\frac{T}{mN}}$ (local variance amplified by heterogeneity). This extends Theorem 1 by incorporating data heterogeneity, showing that non-i.i.d. distributions inflate both upper and lower bounds on error, exacerbating privacy-fairness-utility tradeoffs.

\end{theorem1}
\begin{proof}[Sketch]
    The proof extends Theorem 1's decomposition ($e_{LR} + e_{AR} + e_{\mathcal{G}}$) to non-i.i.d. settings, incorporating heterogeneity measures from FL convergence analyses.

1. Learner Regret ($e_{LR}$): In non-i.i.d., local updates drift due to heterogeneous gradients. Extend Eq. (\ref{eq: LR}) with $\frac{\Gamma T}{N}$ (drift penalty) \cite{mcmahan2017communication}: $e_{LR} = O\left(\frac{B^2 \mathcal{A}^2 \sqrt{T \ln(1/\delta)} \ln m}{\varepsilon_p} + \frac{\Gamma T}{N}\right)$. For lower bound, heterogeneity adds minimal drift $\Omega(\Gamma T / N)$.

2. Auditor Regret ($e_{AR}$): Fairness enforcement via $\lambda$ updates is distorted by heterogeneous attribute distributions. Extend Eq. (\ref{eq: AR}) with variance term: $e_{AR} = O\left(\frac{B^2 \varepsilon_f^3 T^2 \ln(1/\delta) \ln 2^{|\mathcal{A}|} \log(|\mathcal{K}|+1)}{\varepsilon_p^2} + \sigma^2 \sqrt{\frac{T}{mN}}\right)$. Lower bound mirrors with $\Omega(\sigma^2 \sqrt{T/(mN)})$.

3. Generalization Error ($e_{\mathcal{G}}$): Retain Rademacher-based bound from Eq. (\ref{eq: GE}), but non-i.i.d. increases effective complexity via variance, lower bound follows similarly with added heterogeneity terms.

Summing yields the bounds. The non-i.i.d. terms inflate error (upper/lower), implying worse tradeoffs, e.g., stricter $\varepsilon_f$ or smaller $\varepsilon_p$ amplifies drift, requiring more clients $N$ or rounds $T$ to mitigate. Concentration holds w.p. $1 - \beta$ via McDiarmid's inequality \cite{agarwal2018reductions,mozannar2020fair}.
\end{proof}

\begin{proposition}[Convergence of $(f_t,\lambda_t)$ to Saddle Region]
\label{prop: convergence_saddle}
Consider the learner-auditor game on the convex surrogate objective with bounded dual domain $\|\lambda\|_1\le B$. If the learner uses no-regret updates for $f_t$ and the auditor uses exponentiated-gradient (or projected subgradient) updates for $\lambda_t$, then both players have $O(1/\sqrt{T})$ average regret. Consequently, the averaged iterates $(\bar f_T,\bar\lambda_T)$ satisfy
\[
\max_{\lambda\in\Lambda}\mathcal{L}(\bar f_T,\lambda)-\min_{f\in\mathcal{F}}\mathcal{L}(f,\bar\lambda_T)=O(1/\sqrt{T}),
\]
which implies convergence to an $\epsilon_T$-saddle point with $\epsilon_T\!\to\!0$ as $T\!\to\!\infty$ \cite{agarwal2018reductions,jagielski2019differentially}.
\end{proposition}

\begin{proof}[Sketch]
The result follows from standard online convex optimization arguments: no-regret guarantees for both primal and dual players imply vanishing primal-dual gap for averaged iterates in zero-sum games. The bounded dual domain controls the auditor's update norm, yielding the $O(1/\sqrt{T})$ rate.
\end{proof}





\subsection{Convergence and Robustness}\label{se: Convergence and Robustness}
We analyze robustness to distribution shifts between sensitive groups ($p$) and protected groups ($\hat{p}$) using the Total Variation (TV) distance, i.e. $TV(p, \hat{p})$, which measures distribution divergence between $p$ and $\hat{p}$.


\begin{theorem1}[Fairness Discrimination Bound of \textit{FedPF}]
The total fairness discrimination across all clients is bounded by:

\begin{equation}
    \small
         \sum_{i = 1}^{\mathcal{N}} (\gamma_{y, \hat{p},i} (f_i) - \gamma_{y, p, i}(f_i) ) \leq \mathcal{N} \cdot \alpha_{\max},
    \end{equation}

where $\alpha_{\max} = \max_{i \in \mathcal{N}} \{ TV\left(p_i, \hat{p}_i\right)\}$.
    \label{th: fairness_constraints}
\end{theorem1}

\begin{proof}[Sketch]
    For any group label $a \sim p$, $a' \sim \hat{p}$ of client $i$, from Theorem 1 in work \cite{wang2020robust}, we have:
    \begin{equation}
        \small
        \begin{aligned}
        \sum_{i = 1}^{\mathcal{N}} \gamma_{y, a} (f_i) &= \sum_{i = 1}^{\mathcal{N}} \left\{ \gamma_{y, a} (f_i) - \gamma_{y, a'}(f_i) + \gamma_{y, a'}(f_i) \right\}, \\
        & \leq \sum_{i = 1}^{\mathcal{N}} \left\{ | \gamma_{y, a} (f_i) - \gamma_{y, a'}(f_i) | + \gamma_{y, a'}(f_i) \right\}.
        \end{aligned}
    \end{equation}

    From the Kantorovich-Rubenstein \cite{edwards2011kantorovich}, we obtain:

\begin{equation}
    \small
    \left|\gamma_{y, a}(f)-\gamma_{y, a'}(f)\right|=\left|\mathbb{E}_{a \sim p}[f(\theta)]-\mathbb{E}_{a' \sim \hat{p}}[f(\theta)]\right| \leq T V\left(p_j, \hat{p}_j\right) .
\end{equation}

Therefore, $\sum_{i = 1}^{\mathcal{N}} |\gamma_{ij} (f_i) - \hat{\gamma}_{ij}(f_i) |$ has a deterministic upper bound if $TV\left(p_i, \hat{p}_i\right) \leq \alpha_i$ for each client $i \in N$, where the parameter $\alpha_i$ is a constant. Assume that the parameter $\max_{i \in \mathcal{N}} \alpha_i = \alpha_{\max}$, then we have the upper bound of the distance  $\sum_{i = 1}^{\mathcal{N}} (\mathcal{G}_{ij} (f_i) - \hat{\mathcal{G}}_{ij}(f_i) )$, as follows:

\begin{equation}
     \sum_{i = 1}^{\mathcal{N}} (\gamma_{y, \hat{p},i} (f_i) - \gamma_{y, p, i}(f_i) ) \leq \sum_{i = 1}^{\mathcal{N}} TV\left(p_i, \hat{p}_i\right) \leq \mathcal{N} \cdot \alpha_{\max}.
\end{equation}

where $\alpha_{max} = \underset{i \in \mathcal{N}}{\max}\{ TV\left(p_i, \hat{p}_i\right)\}$.
\end{proof}

\end{document}